\documentclass[smallauthors,authorwidth=0.36\linewidth]{paper}
\date{}

\usepackage{microtype}
\usepackage{graphicx}
\usepackage{subcaption}
\usepackage{booktabs} %

\def\website{%
    \url{https://empowerment-geometry.github.io/scaling-empowerment/}%
}

\usepackage{amsmath,amsfonts,bm}

\newcommand{\figleft}{{\em (Left)}}
\newcommand{\figcenter}{{\em (Center)}}
\newcommand{\figright}{{\em (Right)}}
\newcommand{\figtop}{{\em (Top)}}
\newcommand{\figbottom}{{\em (Bottom)}}

\def\eqref#1{equation~\ref{#1}}

\def\1{\bm{1}}

\DeclareMathAlphabet{\mathsfit}{\encodingdefault}{\sfdefault}{m}{sl}
\SetMathAlphabet{\mathsfit}{bold}{\encodingdefault}{\sfdefault}{bx}{n}

\def\gE{{\mathcal{E}}}

\newcommand{\pie}{\gls{pi_mathrmeff}}
\newcommand{\pip}{\gls{pi_mathrmpot}}
\newcommand{\gEe}{\gls{ge_mathrmeff}}
\newcommand{\gEp}{\gls{ge_mathrmpot}}

\newcommand{\E}{\mathbb{E}}

\DeclareMathOperator*{\argmax}{arg\,max}
\DeclareMathOperator*{\argmin}{arg\,min}

\providecommand{\EE}{\mathbb E}

\providecommand{\cA}{\mathcal A}

\providecommand{\cE}{\mathcal E}

\providecommand{\cL}{\mathcal L}

\providecommand{\cS}{\mathcal S}

\providecommand{\cZ}{\mathcal Z}

\usepackage[glossaries]{preamble}
\usepackage{utils}

\usepackage{placeins}
\usepackage{amsthm}

\providecommand{\dmy}[1]{{\color{gray}#1}}

\definecolor{mypurple}{HTML}{7A4988}

\usepackage{wrapfig}
\usepackage[most]{tcolorbox}
\usepackage{hyperref}
\usepackage[capitalize,nameinlink]{cleveref}
\usepackage[acronym]{glossaries}
\makeatletter
\providecommand{\theHALG@line}{}
\renewcommand{\theHALG@line}{\thealgorithm.\arabic{ALG@line}}
\makeatother

\title{Scaling Empowered Agents}

\author{%
  Catherine Ji \\
  Department of Physics\\
  Princeton University\\
  \texttt{cj7280@princeton.edu} \\
  \And
  Vivek Myers \\
  Department of Computer Science \\
  UC Berkeley\\
  \texttt{vmyers@berkeley.edu} \\
  \AND
  Joshua Lin\textsuperscript{*} \\
  Department of Mathematics\\
  Princeton University \\
  \And
  Alexandra Wingate\textsuperscript{*} \\
  Department of Mathematics\\
  Princeton University \\
  \AND
  Sergey Levine \\
  Department of Computer Science\\
  UC Berkeley \\
  \And
  Benjamin Eysenbach \\
  Department of Computer Science\\
  Princeton University
}

\DeclareRobustCommand{\minoredit}[2]{%
  {\color{yellow!45!black}#1}%
  \if\relax\detokenize{#1}\relax\else
    \if\relax\detokenize{#2}\relax\else
      \ifmmode\;\else\space\fi
    \fi
  \fi
  {\color{green!45!black}#2}}

\providecommand{\vivek}[2][]{{\protect\color{violet}{[\textbf{Vivek\if\relax#1\relax\else(#1)\fi:} #2]}}}
\newcommand{\methodName}{SEE\xspace}

\begin{document}
\maketitle
\begingroup
\renewcommand{\thefootnote}{*}
\footnotetext{Equal contribution.}
\endgroup

\begin{abstract} %
    Empowerment measures the capacity of an agent to actively control its environment.
    In this work, we explain the connection between skill learning and empowerment.
    While skill learning is not empowerment, we introduce a simple scalable recipe for empowerment maximization: empowerment corresponds to the number of executable skills from a state, so maximizing empowerment corresponds to reaching states that are good starting points for many skills.
    Our recipe applies to arbitrary underlying skill learning methods.
    Whereas most prior work on empowerment looks at tabular settings, this recipe enables us to study empowerment at a scale where we can realize theoretically-predicted but empirically-elusive properties of empowerment (e.g., gathering resources, environment modification, survival).
    Experiments show that empowerment maximization unlocks interesting emergent properties, such as building a house in Crafter.\footnote{Website and code at \url{https://empowerment-geometry.github.io/scaling-empowerment/}.}
\end{abstract}

\section{Introduction}
\label{sec:intro}

\textbf{Empowerment}, a measure of agency, has been speculated as a unified mathematical explanation for myriad properties of intelligent agents, ranging from building causal models~\citep{yiu2026empowerment,abel2025plasticity,lidayan2025intrinsicallymotivated},  acquiring resources~\citep{mohamed2015variational}, maintaining homeostasis~\citep{klyubin2008keep,salge2014empowerment}, to explanations of how children explore~\citep{lidayan2025intrinsicallymotivated}.
Prior work argues that empowerment should be especially important in open-ended environments, but does not maximize empowerment from scratch in such settings ~\citep{song2026estimating, lidayan2025intrinsicallymotivated}.
Despite decades of research studying the behavior of empowerment-maximizing agents, there has not yet been empirical demonstrations of algorithmic empowerment maximization in the ``big world''~\citep{javed2024big} -- challenging, large settings where behaviors like survival and resource collection become tough.
It is only in such tasks where we expect to see rich emergent behaviors from empowerment maximization.

Agents should successfully interact with an environment that they cannot fully describe. 
For example, even if a young child cannot fully explain all of the intricacies of a game like Minecraft, the child will still become empowered by defeating enemies and building a house, reaching an empowered state from which they know how to do many things. However, maximizing empowerment is a tricky computational problem. First, empowerment maximization requires solving for empowerment at particular contexts. If the context is a state, this means solving for empowerment in all of reachable state space. Second, solving for the empowerment at a particular context is a challenging problem and requires not only modeling the relationship between actions and outputs (e.g. a transition model), but also finding the actions that maximize control. While tractable using standard algorithms in tabular settings, the combination of these difficulties means canonical methods like Blahut-Arimoto cannot exactly solve this computational problem in high-dimensional settings. In the last decade, methods like \gls{vim}~\citep{mohamed2015variational}, \gls{variationalempowerment}~\citep{karl2017unsupervised}, and \gls{latentgce}~\citep{zhao2021efficient} aimed to address this computational problem using deep learning, and were the first to demonstrate learning-based empowerment maximization in simple tabular and continuous settings. However, there has not yet been a stable empowerment-maximization method that builds on the foundation of these works and scales to large continuous or discrete settings. This gap is important to resolve, given that large settings are precisely the right environments for empowerment maximization -- where an agent must execute intelligent behavior despite having a limited view of a complex world.

In this work, we scale empowerment maximization to a regime where theorized intelligent behaviors can emerge.
Our starting point is the observation that empowerment maximization is closely related to \gls{usd} methods.
The goal of empowerment maximization is to reach states from which it is easy to reach a diverse set of states.
Relatedly, USD methods aim to learn a diverse set of different ``skills''.
However, unlike prior empowerment-maximization methods that have struggled to scale, modern skill discovery methods can learn interpretable, diverse behaviors in continuous, high-dimensional tasks.
As we will show, a particular class of USD methods called \gls{misl} measures the point-wise empowerment: at any particular state, the number of different skills an agent can execute from that state lower bounds its empowerment.
In other words, an agent is empowered at states where it has many options, or skills.
Thus, while skill learning methods do not directly learn policies that maximize empowerment, we can use scalable skill learning methods to estimate the empowerment of a particular state; we can then train another policy that ascends the empowerment landscape to get to states where it has many options.

Our work has two key contributions. \textbf{First}, we provide a general recipe to convert scalable skill-learning methods into empowerment maximization methods and show our methods scale when baselines cannot. \textbf{Second}, we demonstrate our method in challenging, high-dimensional settings where empowerment-maximizing behavior becomes nontrivial, and analyze the behavior of empowered agents including resource gathering, survival, and environment modification. For example, early in training in Crafter, empowerment-maximizing agents learn to build houses at night when enemies appear and make and use tools. \textbf{Our work shows that scalable empowerment maximization is possible, demonstrating interesting emergent behaviors in complex, open-ended setting.}

\section{Related Work}
\label{sec:related}

\paragraph{Intrinsic Motivation Exploration Methods.}  
Intrinsic motivation enables exploration without external
information about the \gls{mdp}~\citep{gruaz2025merits}. Prior intrinsic motivation methods reward curiosity~\citep{pathak2017curiositydriven}, novelty~\citep{bellemare2016unifying}, information gain~\citep{houthooft2016vime}, and prediction error~\citep{burda2019exploration}. 
Empowerment instead rewards agency and control~\citep{mohamed2015variational}, a motivation thought to explain human behavior and desired in complex, open-ended settings ~\citep{gopnik2024empowerment,yiu2026empowerment,dai2025understand,lidayan2025intrinsicallymotivated, yang2026assisting}. Prior works predicts that empowerment maximization should generate behaviors such as (1) environment modification~\citep{salge2014changing}, (2) resource gathering~\citep{mohamed2015variational}, and (3) survival. While these emergent behaviors have been previously demonstrated in simple experiments, we present scalable empirical evidence in large, open-ended settings.

\paragraph{Potential and Effective Empowerment.}
As originally introduced, empowerment quantifies the influence an agent's open-loop actions can exert on future states~\citep{klyubin2005empowerment,salge2014empowerment}. Two complementary forms of empowerment are useful~\citep{capdepuy2011informational, ji2026geometry}. \emph{Potential empowerment} measures the channel capacity between a sequence of actions and future observations conditioned on a starting state, measuring the \emph{maximal possible} control from a particular context. In contrast, \emph{effective empowerment} measures how many bits a particular distribution over open-loop controls communicates to the future. Potential empowerment is a scalar property of the context, while effective empowerment is also a function of the controls.

\paragraph{Prior Deep Learning Methods for Estimating and Maximizing Potential Empowerment.}
\gls{vim} maximizes a variational lower bound on open-loop empowerment by learning a decoder that infers action sequences from policy outcomes \citep{mohamed2015variational}. Later empowerment maximization works like \gls{latentgce} alternatively learn latent forward models to estimate the same empowerment, but relies on sufficient coverage over actions sequences in high-empowerment regions ~\citep{zhao2021efficient}. This requirement becomes increasingly demanding in higher-dimensional tasks.
In contrast, we study \emph{closed-loop} empowerment \citep{gregor2017variational, tiomkin2025process}, which is the channel capacity of an information channel from a collection of closed-loop policies to outcomes. Thus, a state is empowering if an agent can execute a diverse repertoire of reactive policies from that state. 
This closed-loop formulation means empowerment measures what a standard \gls{rl} agent can actually do and naturally connects with unsupervised skill-discovery methods.

\paragraph{Unsupervised Skill Discovery.}
\gls{misl} methods aim to learn diverse policies that maximize a \gls{mi} between skill latents $z \in \cZ$ and the outcomes of latent-conditioned policies (i.e. $\pi(a \mid s,z)$).
Modern \gls{misl} methods learn policies that optimize approximations to \gls{mi} objectives over one-step transitions $I(Z;S,S')$ (\citep{zheng2025misl, park2024metra, park2026learning}, to second order for METRA's representation objective), long-horizon transitions $I(Z;S_\tau \mid S_0)$ for termination time $\tau$~\citep{gregor2017variational}, occupancies $I(Z;S)$, \citep{eysenbach2019diversity} and the conditional \gls{mi} $I(Z;S' \mid S)$ \citep{sharma2020dynamicsaware}. Typically, these methods learn a classifier that identifies skills from outcomes and use this classifier as a reward, so that the expected value function is a variational lower bound on the \gls{mi}. These methods scale to high-dimensional locomotion and pixel-based tasks.

\paragraph{Skill Learning and Effective Empowerment.}

The \gls{mi} maximized in \gls{misl} methods are related to closed-loop effective empowerment. Skills are \emph{not} potential empowerment-maximizing policies, but they can provide estimates of the potential empowerment.
Furthermore, while recent works have called attention to this relationship between skill-learning and empowerment~\citep{levy2025representation}, to our knowledge, there has not yet been work that exploits this connection to do scalable potential empowerment maximization. 
We close this gap by adapting modern \gls{misl} methods to maximize potential empowerment in large, open-ended settings.

\section{Defining Empowerment}

\label{sec:prelims}

Our work focuses on a controlled Markov process (i.e., an \gls{mdp} without a reward function) with states
$s \in \cS$, actions $a \in \cA$, and environment dynamics
$p(s' \mid s,a)$.
Our goal is to learn a \emph{potential empowerment} policy $\pip$ that visits states with high potential empowerment $\gEp(s)$:
\begin{equation}
        \pip
        \in
        \argmax_{\pi \in \Pi}
        \E_{\pi}\left[
            \sum_t \gamma^t \gEp(s_t)
        \right].
    \label{eq:pot_goal}
\end{equation}

Empowerment $\gEp(s_0)$ is typically~\citep{salge2014empowerment} defined as the maximal conditional \gls{mi} between behavior $Z$ and future states $S_+$, given an initial state $s_0$: $\gEp(s_0) = \max_{p(A)}I(A; S_+ \mid S_0 = s_0)$. However, prior work differs on what exactly constitutes the behavior $Z$: is it the immediate next action~\citep[Sec.~4.4.2]{salge2014empowerment}, a sequence of actions~\citep{klyubin2005empowerment}, or the index of a closed-loop policy~\citep{gregor2017variational,tiomkin2025process}? Prior work also differs in how precisely the future state $S_+$ is sampled. This paper will adopt an expressive definition of empowerment that measures how behaviors described by closed-loop policies control states \emph{many} steps into the future.

To make the definition of behavior precise, we will consider a $z$ and context-conditioned policy $\pi(a \mid s, z; s_0) \in \Pi_\cZ$, with $s_0$ taken to be the starting context (in our case, a starting state $s_0$) of the behavior. For a specific ``skill vector'' $z$ we will refer to this policy as a ``skill''. Importantly, the context state $s_0$ and the current state $s$ are two different objects -- the context determines the ``start point'' of the skill, while current state tells the agent what to do now.

Let $p(z)$ be a distribution over these skills, which we will assume to be uniform (without loss of generality~\citep{ji2026geometry}).
By introducing $p(z)$ as a distribution over skill vectors, we define a distribution over the space of policies. The term ``skill'' comes from the fact that this distribution over the space of policies may assign zero probability to almost all policies, so the skills are the set of policies with non-zero probability~\citep{eysenbach2022information}.
For each of these skills $\pi(a \mid s, z)$, we define the \gls{dsom} as:
\begin{equation}
    p^\pi(s_+ = s \mid s_0,z)
    \triangleq
    (1-\gamma)\sum_{t=0}^{\infty}
    \gamma^t p^\pi(s_{t+1}=s \mid s_0,z).
    \label{eq:dsom}
\end{equation}
Averaging over all skills gives
\begin{equation}
    p^\pi(s_+ \mid s_0)
    =
    \E_{p(\dmy{z})}
    \left[p^\pi(s_+ \mid s_0,z)\right].
    \label{eq:dsom_marginal}
\end{equation}

We now define the \emph{potential} empowerment using the \gls{dsom} definitions:
\begin{definition}[Potential empowerment]
    Let $p(z)$ be given, and let $S_+ \sim p^\pi(\dmy{s_+}\mid s_0,z)$ be a sample from the \gls{dsom} of skill $\pi(a \mid s, z)$. Then, the potential empowerment of state $s$ is the conditional \gls{mi} between the skill vector $z$ and the future state $S_+$:
    \begin{equation}
        \gEp(s_0)
        \triangleq
        \max_{\pi \in \Pi_\cZ}
        \underbrace{I^\pi(Z;S_+ \mid S_0=s_0)}_{\triangleq \gEe(\pi, s_0)}.
    \end{equation}
    \label{def:pot}
\end{definition}
The potential empowerment can be understood as the maximum number of bits that the skill vector $Z$ can broadcast into the future ($S_+$), starting from state $s_0$. Our goal is to learn a policy $\pip$ that uses this empowerment as a reward, effectively navigating to states where it can broadcast many bits into the future (i.e., exert a strong influence on future outcomes).

Estimating potential empowerment requires solving the inner maximization problem over $\Pi_\cZ$. We will refer to the objective for this inner optimization problem as \emph{effective empowerment}, $\gEe(\pi, s_0)$.
\textbf{We observe that this inner optimization problem is similar to that solved in prior works on \gls{misl}}~\citep{gregor2017variational,zheng2025misl}. Whereas potential empowerment  measures how good a \emph{state} is, effective empowerment measures how good a \emph{set of policies} is from a particular state.
Of course, if this inner optimization problem is only solved approximately, then we obtain a lower bound on the empowerment.

\section{Skills Estimate Empowerment (\methodName): A Scalable Algorithm for Maximizing Empowerment}

\label{sec:misl}
In this section, we introduce a scalable, online algorithm (\methodName) to maximize empowerment (Eq.~\ref{eq:pot_goal}) by exploiting the connection between skill learning and empowerment. 
This connection between learning skills (Sec.~\ref{sec:learn_skills}) and maximizing potential empowerment (Sec.~\ref{sec:prelims}) suggests a two-policy algorithm.
The skill policy $\pie$ estimates a lower bound $\widehat{\gEp}(s)$ on the state-wise potential empowerment, while the potential policy $\pip$ maximizes this estimate as an intrinsic reward.
The effective and potential empowerment policies support one another: the skills $\pie$ measure the controllable futures available from a state, while the potential empowerment policy $\pip$ moves the agent toward states from which the agent can exert control. 

\subsection{The Inner Optimization: Learning Skills}
\label{sec:learn_skills}

We modify prior \gls{misl} methods~\citep{zheng2025misl} to maximize effective empowerment, which lower bounds potential empowerment by \cref{def:pot}. Recall that the expectation of a classifier that identifies skills from transitions over the course of the trajectory is a variational lower bound on the \gls{mi}. This mathematical observation implies that setting the reward to the skill classifier gives a value function that, in expectation, is a lower bound on the \gls{mi}. \methodName builds on \gls{csf}~\citep{zheng2025misl}, a recent state-of-the-art skill-learning method, to approximately lower bound the effective empowerment $\gEe(\pi, s_0) = I^\pi(Z; S_+ \mid S_0 = s_0)$ for a fixed starting state $s_0$. \gls{csf} estimates the one-step \gls{mi} $I(Z;S,S')$ from a classifier that identifies continuous skills from one-step transitions $s \rightarrow s'$. To capture long-horizon dependence between skills and future states, we instead train the classifier on transitions separated by a geometric random variable lengths of time $\Delta \sim \mathrm{Geom}(1-\gamma)$:
\begin{equation}
    f_\phi(s_0,s_+,z)
    - \log \E_{\tilde z\sim p(z)}\!\left[e^{f_\phi(s_0,s_+,\tilde z)}\right]
    \approx \log \frac{p^\pi(z \mid s_0, s_+)}{p(z)}.
    \label{eq:bilinear}
\end{equation}
Following representation parameterizations in \gls{csf}~\citep{zheng2025misl}, \gls{metra}~\citep{park2024metra}, and GenDa~\citep{park2026learning}, \methodName parameterizes the classifier logits $f_\phi$ with a bilinear form of the learned representations $\phi$.
This parameterization means that skills are only distinguishable if learned representations linearly separate over the course of skill execution. Prior work find that this parameterization is important to generate sustained, space-filling behaviors~\citep{park2024metra}. \methodName utilizes \gls{infonce}~\citep{oord2018representation} to estimate the classifier. Additional discussion on this parameterization and its implications are in \cref{app:effective_lb}. 

Accumulating this classifier over the trajectory gives, in expectation, a variational lower bound on the \gls{mi}.
For a skill and start-state conditioned policy $\pi(a\mid s, z; s_0)$, define the effective empowerment Q- and value function as
\begin{align}
Q_{\mathrm{eff}}(s,a; s_0, z)
    &\triangleq
    (1-\gamma)\,
    \E_{\pi(a\mid s, z; s_0)}\left[
    \sum_{k=0}^{\infty}\gamma^k
    \log \frac{q(z \mid s_0, s_{t+k+1})}{p(z)}
    \middle| s_t=s,\; a_t=a
    \right] \label{eq:eff_q} \\
V_{\mathrm{eff}}(s; s_0, z) 
    &\triangleq 
    \E_{\pi(a \mid s,z; s_0)}\left[
    Q_{\mathrm{eff}}(s,a; s_0, z)
    \right].
    \label{eq:eff_v}
\end{align}
Then, in expectation over a uniform skill source $p(z)$ initialized at the start-state $s_0$, we have that
\begin{align}
\EE_{p(z)} [V_{\mathrm{eff}}(s_0; s_0, z)] &\leq \gEe(\pi, s_0),
\label{eq:eff_bound}
\end{align}
so the expected value function lower bounds the effective empowerment. Potential empowerment then uses a Monte-Carlo estimate of $\EE_{p(z)} [V_{\mathrm{eff}}(s_0; s_0, z)] \approx \frac{1}{K} \sum_{k=1}^{K} V_{\mathrm{eff}}(s_0; s_0, z_k)$, with $z_1,\ldots,z_K\overset{\mathrm{iid}}{\sim}p(z)$, to get an estimate of the pointwise potential empowerment:
\begin{equation}
\hat{\gEe}(\pi, s_0) = \EE_{p(z)} [V_{\mathrm{eff}}(s_0; s_0, z)] \leq I(Z;S_+ \mid S_0 = s_0) \leq \gEp(s_0).
\label{eq:pot_reward}
\end{equation}

\subsection{Practical Considerations}
\label{sec:practical}

\paragraph{Conditional and Joint \gls{mi} Classifiers.}
While empowerment is theoretically a conditional \gls{mi}, in practice, we found that 
training policies with joint \gls{mi} classifiers that identify skills from a current state $s_t$ and future state $s_{+}$ \emph{from that state} led to stronger
performance in continuous benchmarks. Otherwise, the classifier aggressively overfit to the context $s_0$. The skill reward still uses the true skill's starting state $s_0$. We include implementations of \methodName
using the joint \gls{mi} classifier and a standard CMI classifier. By default, \methodName uses this joint classifier in continuous settings and the conditional classifier
in Craftax-Classic (\cref{app:data_collection}).

\paragraph{Limiting empowerment context.} 
As noted in \citep{klyubin2005empowerment, abel2025plasticity}, empowerment depends on an agent's particular perception. Otherwise, an agent may obtain high
empowerment by perturbing controllable but meaningless variables.
We therefore run experiments that restrict variables over which empowerment is defined. This restriction determines
which changes in the state are salient to the skill, without restricting
the information available to the policy for control. Thus, the classifier may
receive different degrees of freedom from the full-observation policy and critic.

\paragraph{Two-stage rollout.}
For large state spaces, we cannot simply learn skills for every starting state $s_0$
and then maximize the potential empowerment over the constructed field.
Instead, we consider a two-stage rollout scheme:
within each episode of data collection, we first roll out the potential empowerment policy and then roll out one of the effective empowerment policies (i.e., sample a skill $z$ and execute
$\pie(a \mid s,z;s_0)$) (\cref{app:data_collection,app:pseudocode}).
We can interpret skills as local probes of empowerment, where the potential empowerment policy ascends the empowerment landscape estimated by these probes.

\begin{tcolorbox}[
    colback=green!10,      %
    colframe=green!40!black, %
    rounded corners,
    boxrule=0.5pt,
    left=6pt, right=6pt, top=6pt, bottom=6pt
]

\paragraph{Method Summary.}

We adapt \gls{csf}, a scalable \gls{misl} method, to learn skills, which gives us (a lower bound on) the potential empowerment of states $\gE(s)$.
Then, we train a second potential empowement policy using the estimated potential empowerment as a reward. This policy
navigates to states \emph{from which} it can exert control.
\end{tcolorbox}

\section{Experiments}
\label{sec:experiments}
We (1) show scalability of the method, (2) carefully ablate and analyze all of the salient design decisions (and provide open-source implementations for future development), and (3) interpret the resulting emergent behaviors. 
All training setups and hyperparameters are in \cref{app:cont_details,app:disc_details}.

\begin{figure}[!tbp]
    \centering
    \begin{subfigure}[t]{0.49\linewidth}
        \centering
        \includegraphics[width=0.4\linewidth]{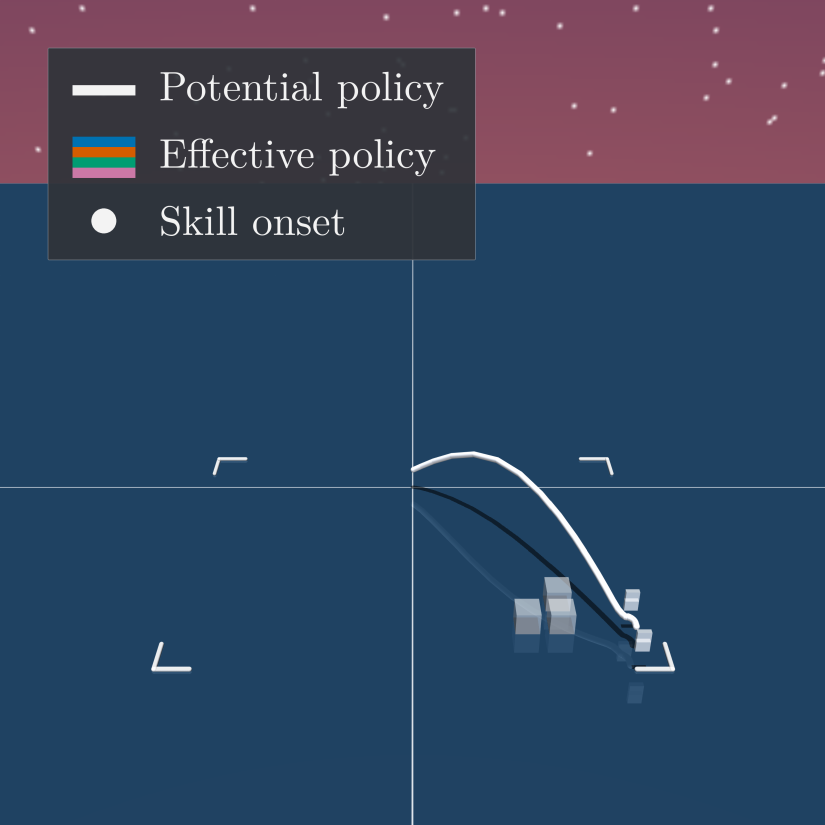}\hspace{0.02\linewidth}
        \includegraphics[width=0.4\linewidth]{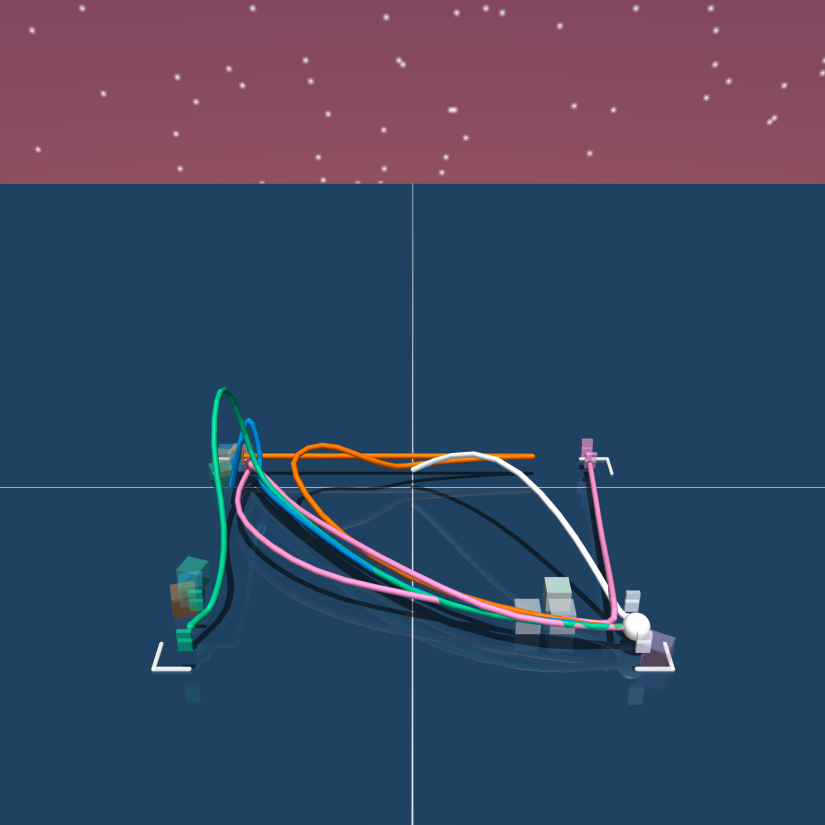}
        \caption{\textbf{BuilderBench cubes.}~\citep{ghugare2026builderbench}}
        \label{fig:highdof_cubes_sub}
    \end{subfigure}
    \hfill
    \begin{subfigure}[t]{0.49\linewidth}
        \centering
        \includegraphics[width=0.4\linewidth]{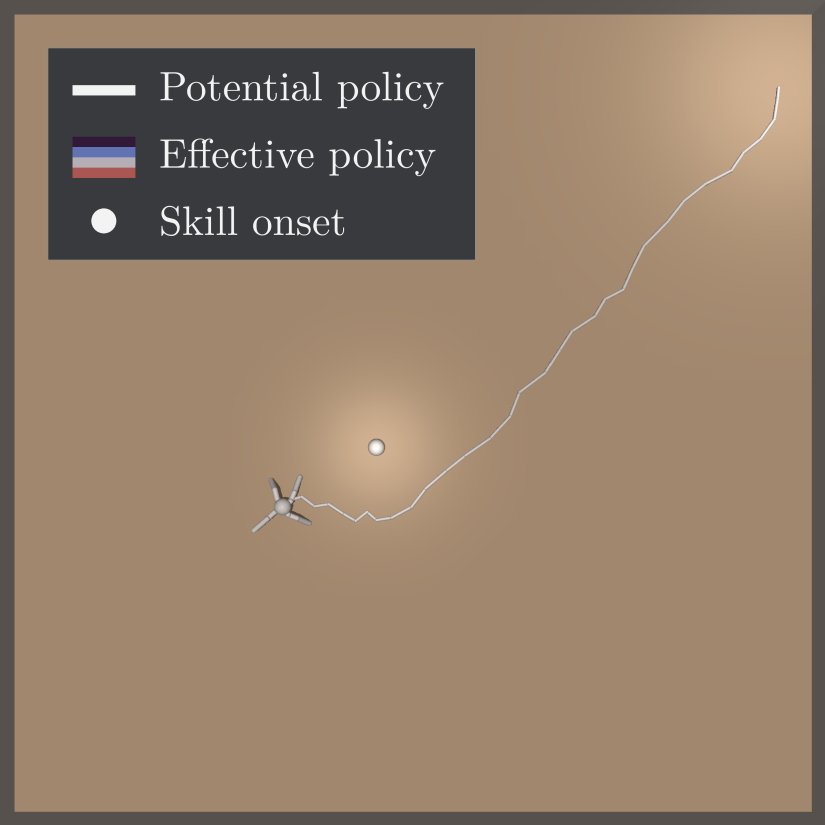}\hspace{0.02\linewidth}
        \includegraphics[width=0.4\linewidth]{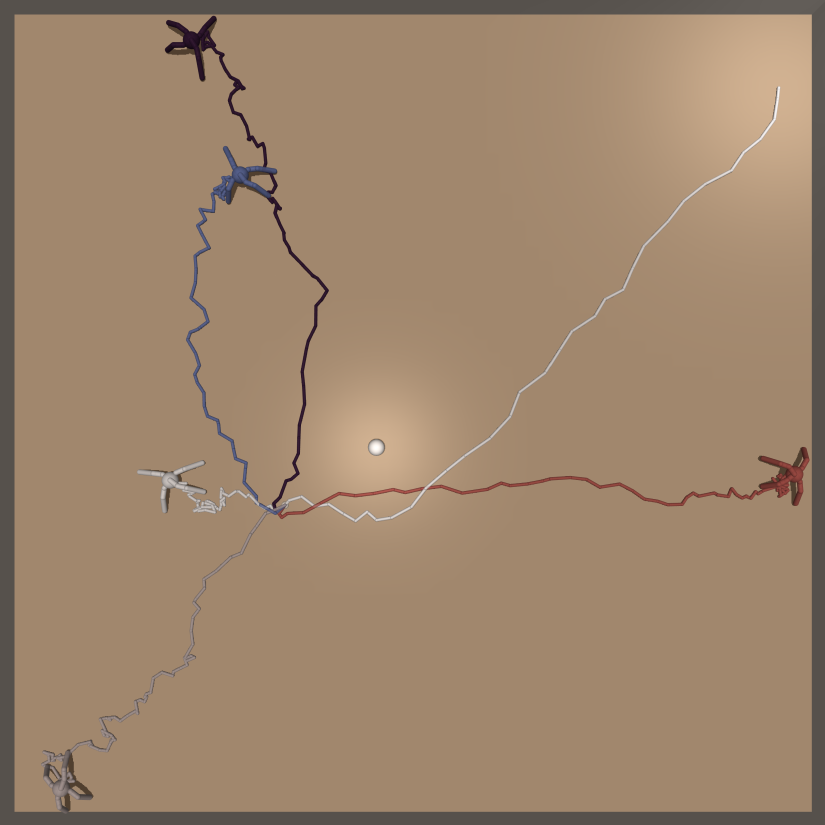}
        \caption{\textbf{Ant locomotion.}~\citep{park2025ogbench}}
        \label{fig:highdof_ant}
    \end{subfigure}
    \par\medskip
    \begin{subfigure}[t]{0.49\linewidth}
        \centering
        \includegraphics[width=0.4\linewidth]{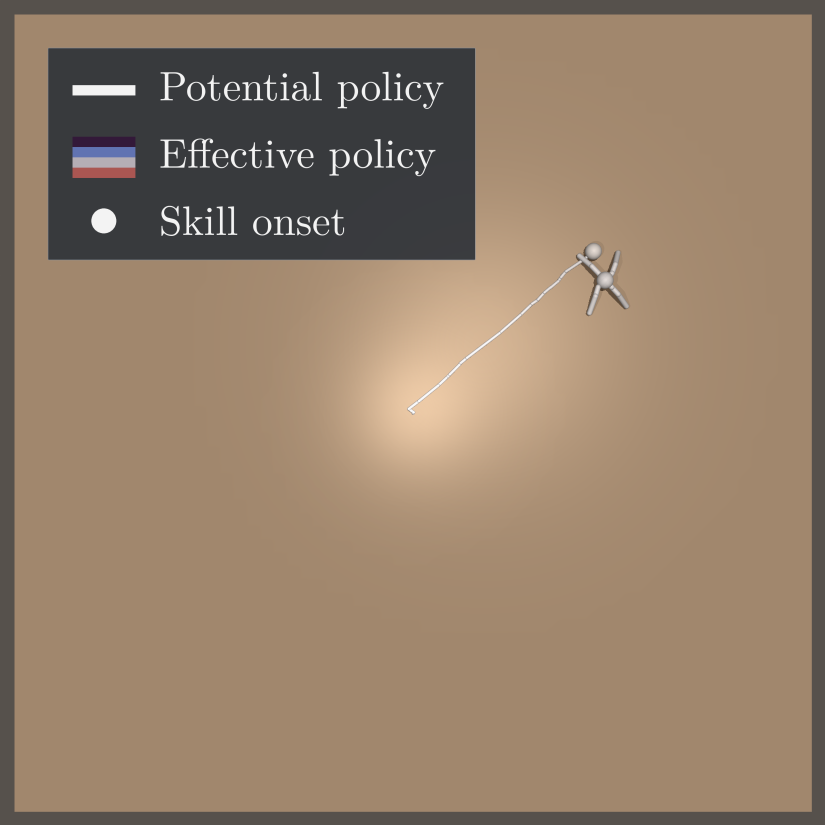}\hspace{0.02\linewidth}
        \includegraphics[width=0.4\linewidth]{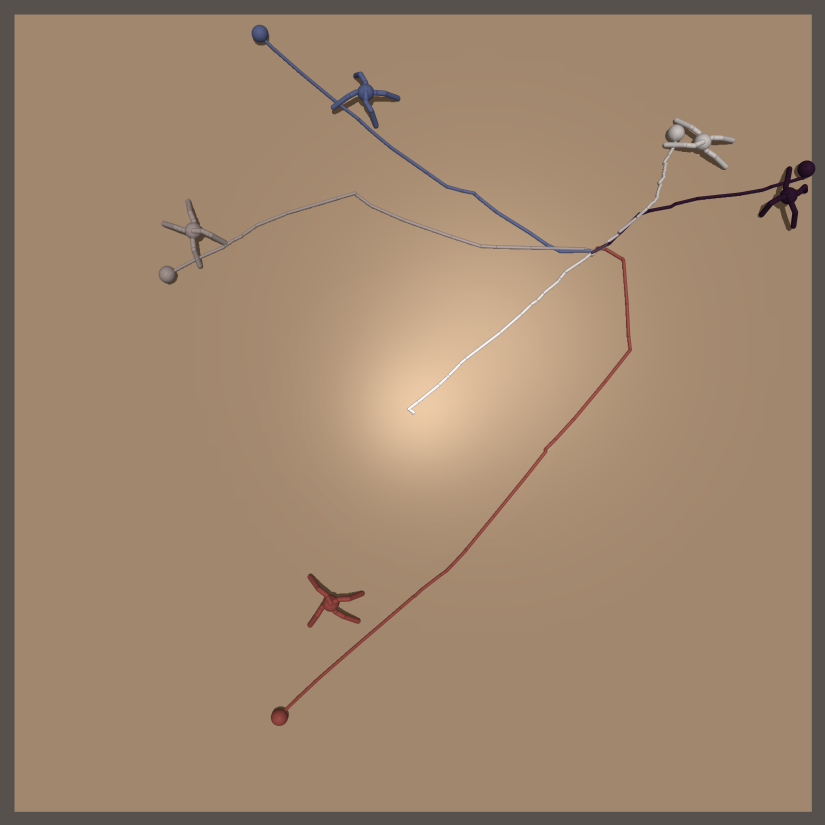}
        \caption{\textbf{Ant soccer.}~\citep{park2025ogbench}}
        \label{fig:highdof_antsoccer}
    \end{subfigure}
    \hfill
    \begin{subfigure}[t]{0.49\linewidth}
        \centering
        \includegraphics[width=0.4\linewidth]{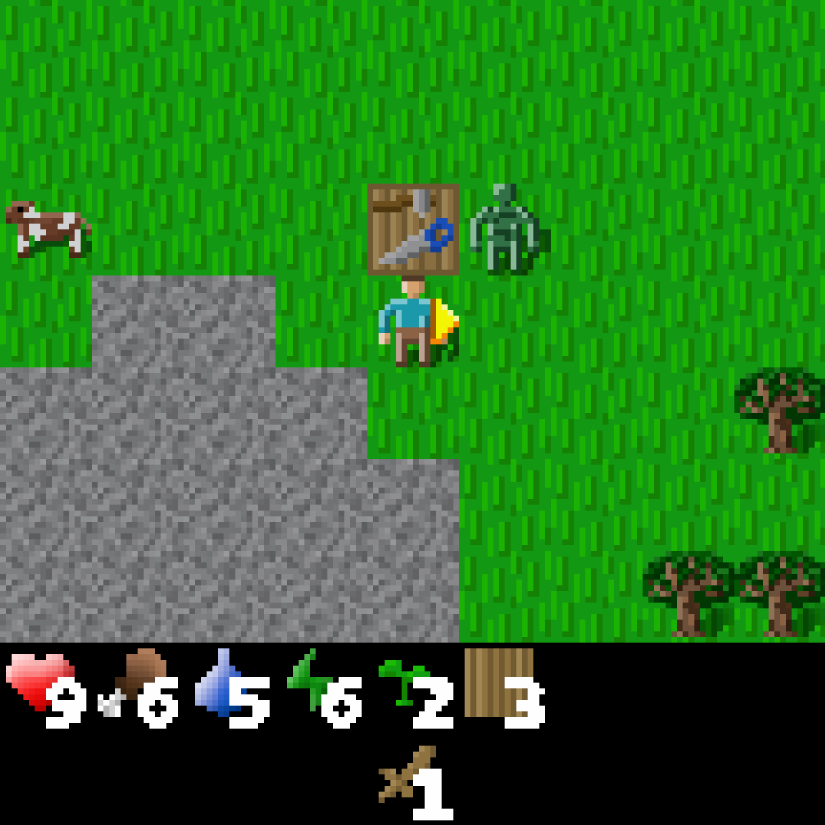}
        \caption{\textbf{Craftax-Classic.}~\citep{matthews2024craftax}}
        \label{fig:highdof_craftax}
    \end{subfigure}
    \caption{\textbf{Empowerment-maximizing behavior across environments.} (a) In the BuilderBench task, a gripper interacts with blocks in an open-ended setting. The potential empowerment policy is visualized with the white rollout, navigating to the blocks and staying there. Then, the effective empowerment policies, or skills, knock the blocks in different directions (colorful lines and blocks). (b) In the first version of the AntSoccer task, an agent can be empowered without interacting with the ball. The potential empowerment policy (white) navigates towards the central area of the arena, while the skills branch out from this point. (c) In the second version of the AntSoccer task, an agent can only be empowered if it controls the ball. The trained potential empowerment policy reaches the soccer ball, while skills kick the soccer ball in different directions. (d) In Craftax-Classic, the potential empowerment agent gathers materials, builds a crafting table, then builds tools using gathered materials.}
    \label{fig:highdof_combined}
\end{figure}

\paragraph{Environments.} To test whether the method works in large, open-ended settings, we implement our method in Craftax-Classic~\citep{matthews2024craftax} (JAX-accelerated Crafter, \cref{fig:highdof_craftax}). Crafter is a large and partially observed setting with a 1345-dimensional symbolic observation and 17 discrete actions consisting of crafting, resource collection, fighting mobs, and eating/drinking. In this environment, we define the empowerment degrees of freedom to be over the agent inventory. The reason for this choice is that we do \emph{not} wish for an agent to be empowered over features like e.g. health, hunger, and thirst, which would imply having a lot of control over these precise status values (for instance, a skill that \emph{decreases} health).

To test whether the method works in continuous settings, we chose two high-DoF benchmarks that capture locomotion and manipulation. The AntSoccer environment consists of an ant and a ball in a closed arena, where an ant can freely kick the ball~\citep{park2025ogbench}. Within the AntSoccer environment, we consider the setting where empowerment is defined over all $42$ degrees-of-freedom (e.g. joint angles, xy position, ball position, ``Ant'') and strictly the ball position (``AntSoccer'').
BuilderBench~\citep{ghugare2026builderbench} consists of a gripper with cubes. We consider an environment with 3 cubes in the corner of a bounded play area (\cref{fig:highdof_cubes_sub}, white corners denote the boxed play area), where empowerment is defined over the $9$ degrees-of-freedom that denote the cube positions.

\paragraph{Baselines.} We consider two  baselines that learn potential empowerment policies: \gls{latentgce}~\citep{zhao2021efficient} and \gls{vim}~\citep{mohamed2015variational}. Both of these methods learn policies that navigate to high open-loop empowerment regions (\cref{app:baselines}). We omit standard \gls{misl} works because they optimize effective empowerment policies from a fixed initial state distribution and do not learn a potential empowerment policy, though these methods bear similarity to our skill-switching ablation.

\paragraph{Ablations.} We ablate key design decisions in \methodName. While our method builds on \gls{misl}, there are fundamental differences between empowerment maximization and skill learning. Skill learning trains effective empowerment policies \emph{usually from a fixed or slightly noised starting state}.
Potential empowerment maximization is \emph{not} as simple as running \gls{misl} out-of-the-box and deals with the tougher problem of learning skills originating over many regions of state space. Design decisions such as skill-switching and the choice of the underlying \gls{misl} parameterization are critical to learn behaviorally interesting policies. We do the following ablations:
\begin{itemize}
    \item \textbf{skill-switching} for learning the potential empowerment policy,
    \item \textbf{joint versus conditional \gls{mi}}, and
    \item the \textbf{bilinear difference-in-representation encoder versus a more general encoder}.
\end{itemize}

\begin{wrapfigure}[17]{R}{0.4\textwidth}
    \centering
\includegraphics[width=\linewidth]{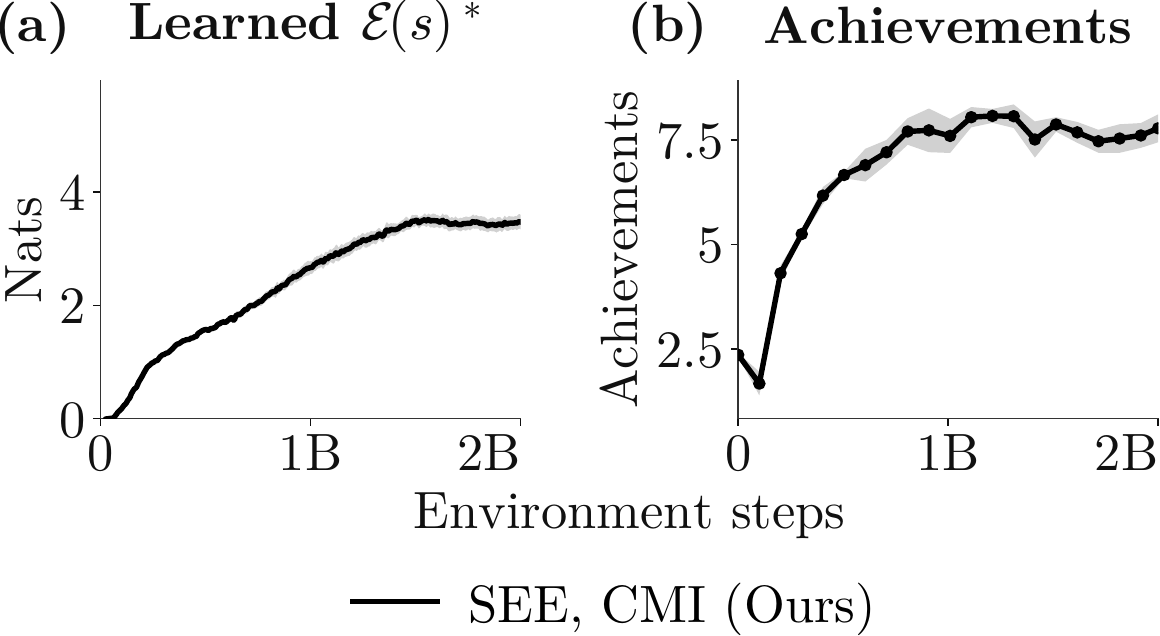}
    \caption{\textbf{The potential empowerment estimate increases with control in Crafter.} (a) The learned empowerment estimate increases consistently during training. $5$ seeds, p.m. standard error after per-seed EMA ($\alpha = 0.05$). (b) Simultaneously, the agent intrinsically learns to complete achievements in the environment \emph{without access to external rewards}. $5$ seeds, evaluated over $128$ random environments, p.m. standard error.}\label{fig:training_curves_craftax}
\end{wrapfigure}
\paragraph{Metrics.} While \gls{mi} is an important signal for training, achieving high \gls{mi} does not guarantee meaningful empowerment maximization. In a continuous setting without constraints or significant noise, prior work has shown that trivial behaviors can realize potentially unbounded \gls{mi} \citep{ji2026geometry}. Thus, we additionally track environment-specific metrics to judge the quality of the learned policies. In Ant, we measure the distance of the ant to the center of the arena. In AntSoccer, we measure the distance of the ant to the soccer ball. In BuilderBench, we measure the distance of the gripper to the cubes. In Crafter, we measure the achievement count. These external metrics measure the quality of learned behaviors \emph{independent} of the \gls{mi} quantities themselves.

\subsection{Scaling and Baselines}

\paragraph{Discrete Setting: Craftax-Classic.}
\methodName scales to high-dimensional, discrete settings.
In Craftax-Classic (\cref{fig:highdof_craftax}), the agent consistently unlocks $7-8$ out of 22 total achievements per rollout from $1 \times 10^{9}$ environment steps.
These metrics exceed the best previously-reported success rates of intrinsic exploration methods in Crafter, which obtains $\sim5.5$ achievements with the same training ~\citep{mohamed2026temporal}.
The estimated empowerment also rises steadily ~(\cref{fig:training_curves_craftax}), tracking the development of interesting behavior.
Implementation details and hyperparameters for Craftax-Classic are in \cref{app:disc_details}.

\paragraph{Continuous Settings: Ant \& AntSoccer.} \methodName learns meaningful empowerment behavior in continuous settings that tracks increasing average empowerment estimates (\cref{fig:continuous_empowerment}).
If the empowerment is defined over all degrees-of-freedom (``Ant''), the policy $\pip$ learns to navigate to the {\it center region of the arena and stay there} (white, \cref{fig:highdof_ant}) while the policies $\pie$ branch and cover the space (colors, \cref{fig:highdof_ant}), following intuition that empowerment is high when the Ant can \emph{do more} away from the walls of the arena.
The baseline potential empowerment policies, in contrast, do not navigate to the center region and barely move at all. When the empowerment is defined only over the ball xy coordinates (``AntSoccer''), the policy $\pip$ navigates close to the ball (white, \cref{fig:highdof_antsoccer}) while the skill policies $\pie$ kick and dribble the ball in different directions (colors, \cref{fig:highdof_antsoccer}). The baseline potential empowerment policies, in contrast, do not consistently navigate close to the ball. When LatentGCE does happen to navigate close to the ball and kick the ball, the agents erroneously kick the ball to the very corner of the arena, where empowerment should be low.

\paragraph{Continuous Settings: BuilderBench.} 
Finally, in BuilderBench, the potential empowerment policy sneaks behind the blocks, where it has easy access to the blocks, and stays there (\cref{fig:training_curves}). The effective empowerment policies then knock the blocks in different directions (\cref{fig:highdof_cubes_sub}), and the arm learns to get closer to the blocks. In contrast, the baseline methods perturb the blocks, but fail to remain close to the blocks at the end of the potential empowerment rollout.

In summary,
\methodName scales to manipulation, locomotion, and open-ended (e.g. Crafter) environments, whereas prior empowerment-maximization methods struggle.

\subsection{Ablating Design Decisions}
\label{sec:ablations}

\paragraph{Policy switching.} During data collection, \methodName switches from the potential empowerment to the effective empowerment policy. An alternative is to separately train potential empowerment and effective empowerment policies.
Despite learning meaningful skills from the starting state, this ablation leads to inconsistent potential empowerment policies that do not reach the center of the arena in Ant nor consistently reach the blocks in BuilderBench. 
Meanwhile, our base \methodName stably learns potential \emph{and} effective empowerment policies in all of these environments.

\begin{figure}[!tbp]
    \centering
    \includegraphics[width=0.6\linewidth]{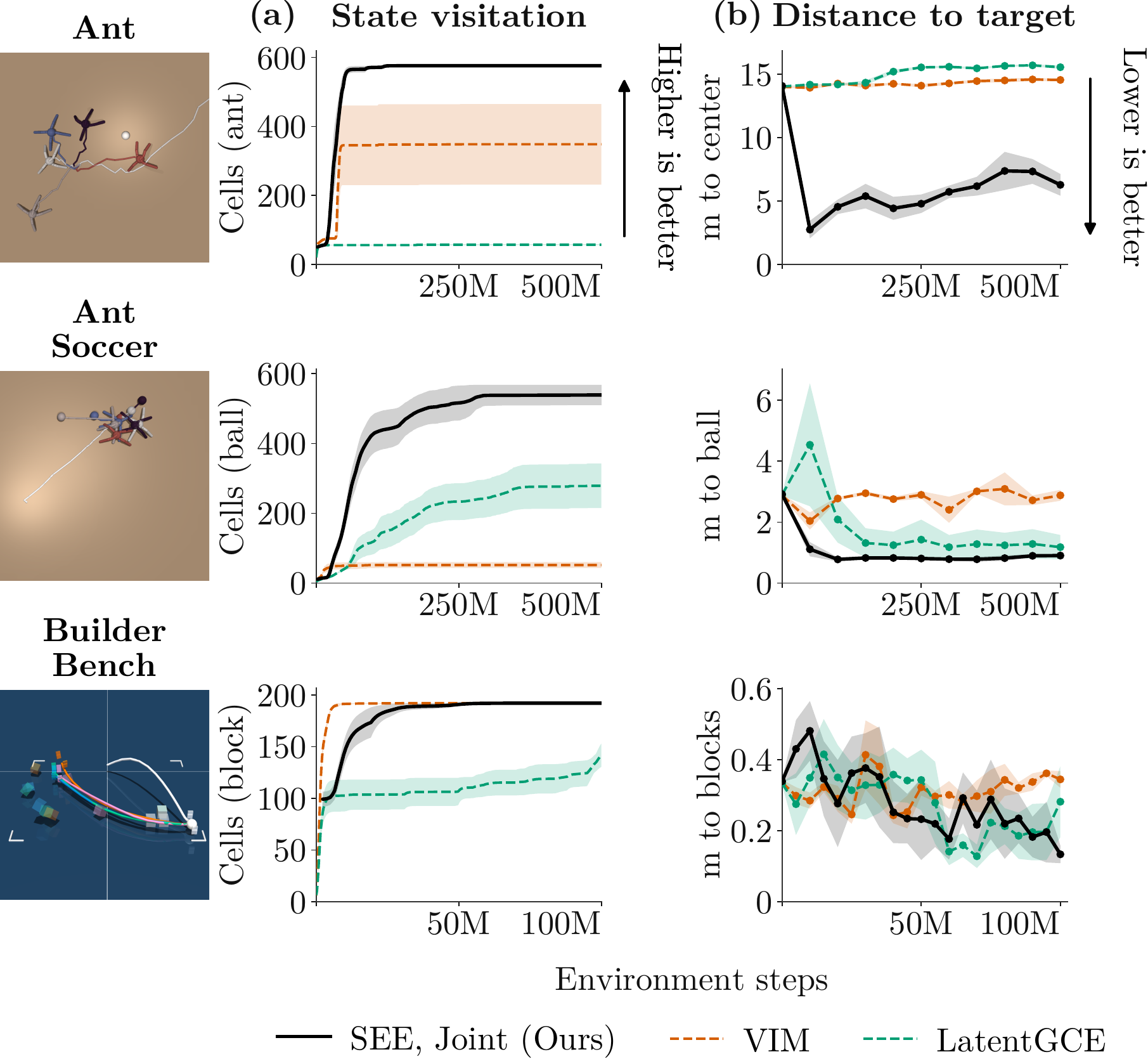}
    \caption{\textbf{Potential empowerment and control increases over training (continuous settings).} We compare \methodName (Black) and baselines VIM (Orange) and LatentGCE (Green). (a) \methodName visits diverse states across all environments, which is necessary to learn correct models and decoders of dynamics and skills. (b) \methodName consistently learns potential empowerment policies that navigate towards central regions of the arena in Ant, to the ball in AntSoccer, and to the blocks in BuilderBench. \gls{vim} is unstable and largely stationary in Ant and AntSoccer. LatentGCE is stationary in Ant. Five seeds for \methodName and LatentGCE and three for VIM, p.m. standard error. Eval metrics averaged over $128$ randomized environments.}
    \label{fig:training_curves}
\end{figure}

\paragraph{Conditional \gls{mi} vs. Joint \gls{mi} Variants.} \citet{klyubin2005all} define empowerment as a conditional \gls{mi} conditioned on a starting context state. However, modern \gls{misl} methods learn variational lower bounds on joint \glspl{mi} (e.g. $I(Z;S,S')$) that do \emph{not} condition on the starting context. We compare using conditional versus joint \gls{mi} parameterizations of empowerment. Further details of how we implement the \acrshort{cmi} version of \methodName is in the Appendix \ref{app:data_collection}.
Policies trained with the true conditional \gls{mi} learn empowering behaviors in AntSoccer and BuilderBench. However, when empowerment is defined over many continuous degrees of freedom like in Ant, the CMI falls into the trap where the decoder and policies overfit to small perturbations, leading to high \gls{mi} values without meaningful behavior. In Crafter, using the CMI leads to the best performance among the ablations and prevents the MI collapse observed when optimizing the joint MI.

\paragraph{Bilinear Difference in Representations vs. General Parameterization.} 

Recent \gls{misl} works such as \citet{park2024metra, zheng2025misl, park2026learning} utilize differences in representations $z^\top (\phi(s') - \phi(s))$ to parameterize classifiers (or similar objects) for skill learning. Replacing this parameterization with a more expressive logit significantly worsens the stability of training and the learned behaviors across environments, as the classifier can closely fit collected data and the policies ``overfit'' to sensitive classifiers.

\subsection{Emergent Behaviors}

\begin{figure}[!tbp]
    \centering
    \includegraphics[width=0.75\linewidth]{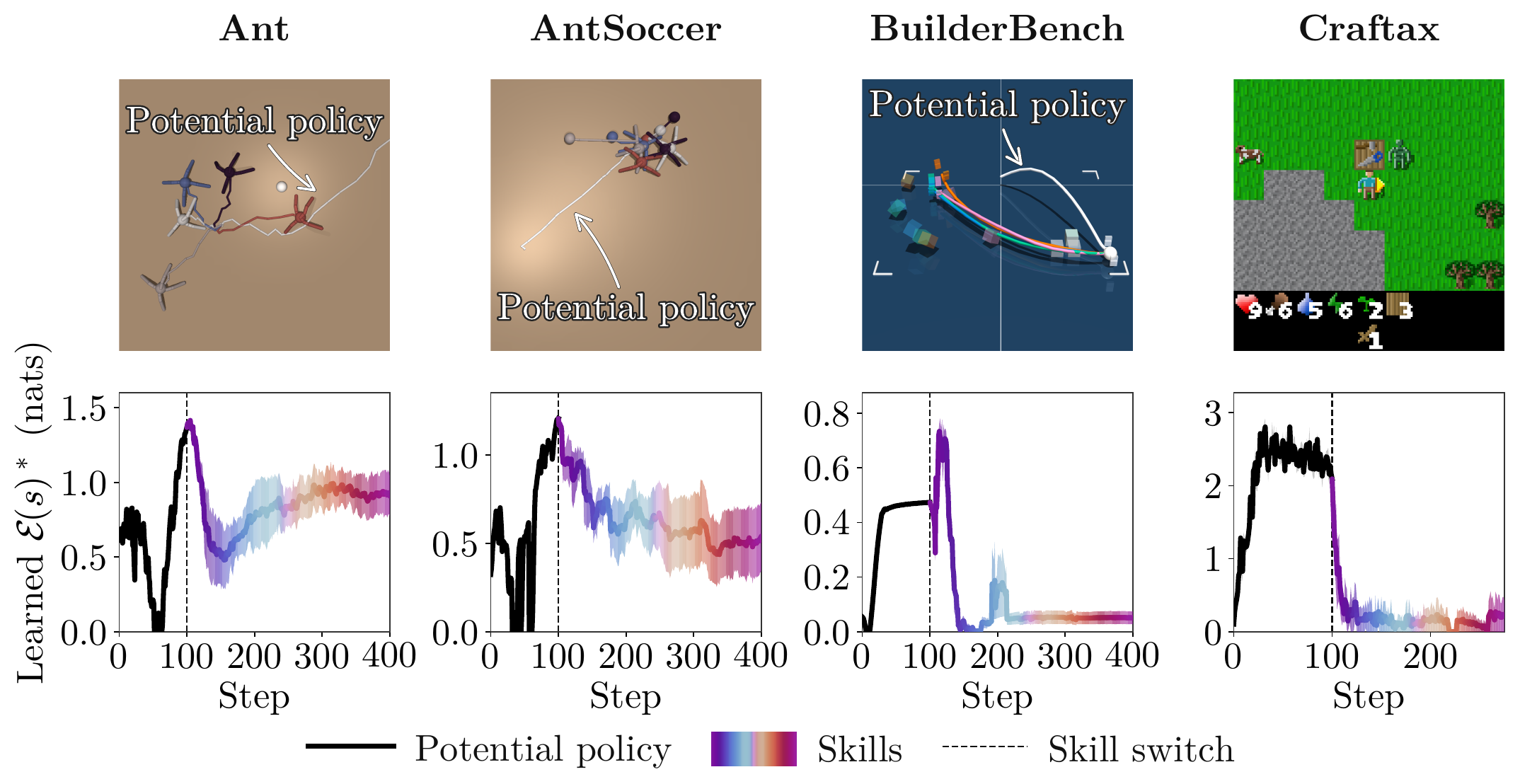}
    \caption{\textbf{Potential empowerment policies accrue empowerment, while skills spend it.} (Top row) Over the course of the trajectory, the potential empowerment policies (white lines) ``gather'' empowerment, reaching empowered states. Correspondingly, the empowerment estimates increase over the course of the trajectory. (Bottom row) At the gray skill switch mark, executing a skill (colorful lines) from the reached state ``spends'' the empowerment, reflecting the intuition that potential empowerment policies collect power and effective empowerment spends power. Traces use one seed and four skill branches per environment.}
    \label{fig:potential_trace_row}
\end{figure}

\paragraph{Accruing and spending power.} Across all environments, the learned switching policies accrue empowerment while executing the potential empowerment policy (\cref{fig:potential_trace_row}, white traces). The policy gathers empowerment by walking to the center of the arena in Ant, getting close to the soccer ball in AntSoccer, getting close to the blocks in BuilderBench, and gathering resources and tools in Crafter. Then, the policies spend the empowerment when executing skills (\cref{fig:potential_trace_row}, colorful traces), spreading out in the arena (Ant), kicking the ball (AntSoccer), and moving the blocks (BuilderBench).

\paragraph{Gathering resources and tool use.} Resource-gathering and tool use enable empowerment-maximizing agents to exert control over the environment. In the Crafter environment, the trained potential empowerment agent learns to acquire many diverse materials (e.g. wood, saplings), such that its potential empowerment increases over the course of the episode. From that potentially-empowered state, the agent executes diverse skills such as building houses, mining, and using the gathered resources to craft new tools (\cref{fig:emergent_behaviors}).

\paragraph{Modifying the environment.} Empowered agents learn to perturb their environment to maximize future control. This phenomena has been documented in prior works~\citep{salge2014changing} in simple tabular settings. In the scaled up Crafter setting, we observe more complex ``building'' behaviors, such as learning to construct a house when defending against enemies (\cref{fig:emergent_behaviors}).

\begin{figure}[htb]
    \centering
    \includegraphics[width=0.6\linewidth]{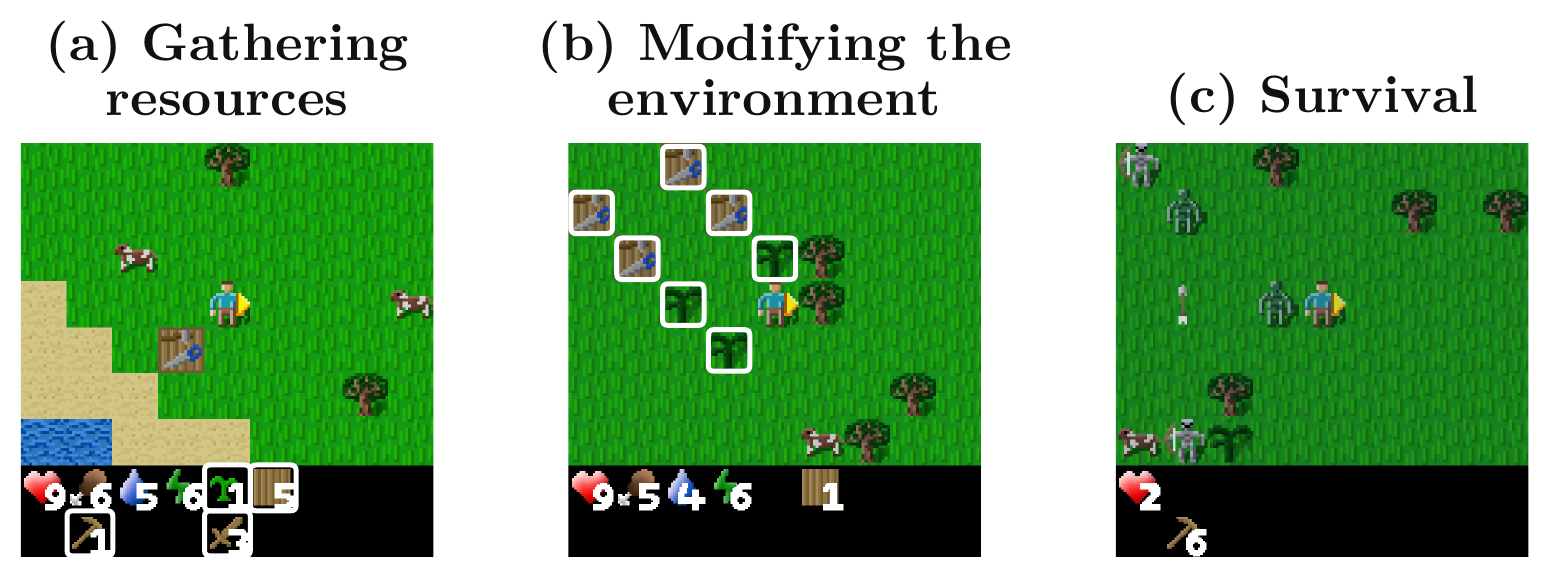}
    \caption{\textbf{Empowerment maximization has interesting emergent behaviors.} (a) In Crafter, the empowered agent quickly accumulates many resources and uses a tool (crafting table, wooden icon with saw) to craft other tools in the game. Resources boxed in white. (b) Early in training, the agent learns to build a ``house'' consisting of plants and crafting tables. Built tiles boxed in white. (c) The agent also naturally learns to survive, avoiding and killing enemies like zombies and skeletons. One seed per example; further qualitative examples across seeds in \cref{app:qualitative}.}
    \label{fig:emergent_behaviors}
\end{figure}

\paragraph{Survival.} In death, an agent no longer has any control and, thus, has no empowerment. Therefore, agents maximizing empowerment should learn survival skills. In Crafter, we see that empowerment-maximizing agents naturally learn survival skills like killing and dodging enemies. Importantly, the empowerment-maximizing agent learns behavior \emph{distinct} from simply training a survival reward (e.g. $r=1$ for every step, $r=0$ upon termination), which does not learn more complex crafting behaviors and \emph{only} prioritizes survival skills.

\section{Discussion}
\label{sec:disc}

Despite the conceptual appeal of empowerment as a tool for understanding natural agents and building artificial agents, its seeming computational intractability has caused prior work to resort to approximations that do not learn one policy that ``keeps its options open''~\citep{klyubin2008keep}.
In providing a conceptual link between skill learning methods and empowerment, we provide a method that maximizes potential and effective empowerment, two policies that measure controllability in distinct ways.
We present key practical considerations to learn empowerment maximization policies in high-DoF settings. 
Future work should investigate simplified parameterizations of empowerment maximization that can resolve practical considerations detailed within our method, at scale.
Furthermore, our work highlights the sensitivity of empowerment maximization to representations of the world, with open questions on the interplay between representations of the world and empowerment maximization~\citep{dai2026worth}.

\paragraph{Acknowledgments.} We would like to thank David Abel, Elom Amematsro, Mahsa Bastankhah, Caleb Chuck, Max Kleiman-Weiner, Andrew Levy, and anonymous reviewers for their insight and helpful feedback. This material is based upon work supported by Google and the National Science Foundation Graduate Research Fellowship Program under Grant No. DGE-2444107 and Award No. 2441665. Any opinions, findings, and conclusions or recommendations expressed in this material are those of the author(s) and do not necessarily reflect the views of the National Science Foundation.

\subsection*{AI use statement}
We used generative AI tools to assist with implementing experiment code, preparing and revising figures, minor writing edits, assembling appendix details, and manuscript formatting. We did not use AI for the methods development, theoretical results, or any major writing. The authors take responsibility for the final content of this work.

\subsection*{Reproducibility statement}
We provide algorithm pseudocode and implementation details in the appendix. 
Code, videos, and additional figures are available at \website.

\appendix
\newpage
\newglossaryentry{gep}{%
    sort={gep},
    name={\ensuremath{\gEp}},
    description={%
        The potential empowerment at state $s$ and iteration $n$, defined as the channel capacity between skills and future states from that state.
    }
}

\section{Algorithm and Implementation Details}
\label{app:alg_details}

Below, we include key algorithm and implementation details. While we describe the Joint MI classifier form of the \methodName for simplicity, we detail the Conditional MI variant of \methodName later in \cref{app:data_collection}. When the skill is independent of the starting state (e.g. drawn from the same prior distribution at a starting state), the joint ($I(Z;S_+, S_0)$) and conditional \gls{mi} ($I(Z;S_+ \mid S_0)$) are equivalent, as $p(z \mid s_0) = p(z)$. Practically, our Joint MI classifier breaks this independence assumption by sampling a state \emph{within a trajectory} as a starting point, then sampling a geometric future state from that point. Further discussion is in \cref{app:data_collection}.

\subsection{Training the classifier}
\label{app:effective_lb}

\methodName utilizes \gls{misl} methods to lower bound the pointwise potential empowerment. We take an approach similar to \gls{csf}~\citep{zheng2025misl}, employing noise contrastive estimation to learn the skill classifier, which provides a variational lower-bound on the \gls{mi}.

Without loss of generality, assume that the skills $z$ are drawn from a uniform distribution on a sphere $p(z)=\mathrm{Unif}(\mathbb S^{d-1})$. Let $f_\phi (s,s_+, z)$ denote a parameterized logit over a state $s$, future state $s_+$, and skill $z$, constrained to take the form
\[
f_\phi (s, s_+, z) = \frac{1}{\tau} z^\top (\phi(s_+) - \phi(s))
\]
which we refer to as a bilinear difference-of-representations parameterization, with temperature $\tau=0.1$. This parameterization of the logit appears in recent SOTA \gls{misl} literature and has been shown to train persistent skill behavior that spread over latent space; we refer readers to these references for more information \citep{zheng2025misl,park2024metra,park2026learning}.

Then, given access to a dataset $\mathcal{D}$ of starting states, skills, and outcomes $(s, s_+, z)$, where outcomes $s_+ \sim p^\pi(\cdot\mid s,z)$ are drawn from the \gls{dsom},
InfoNCE trains a classifier to distinguish whether skills are drawn from the joint distribution $p(z, s, s_+)$ or the product of marginals $p(z)p(s, s_+)$. For a batch of size $B$ with one positive pair $(\textcolor{teal}{z, (s, s_+)})$, where $\textcolor{teal}{z} \sim p^\pi(\cdot \mid \textcolor{teal}{s,s_+})$, and $B-1$ negative pairs $(\textcolor{brown}{z}, \textcolor{teal}{(s, s_+)})$, where $\textcolor{brown}{z} \sim p(\dmy{z})$, we have:
\begin{align*}
    \text{InfoNCE Loss: } \quad \cL(\phi) = \frac{1}{B}\sum_{i=1}^{B}\left[\log\sum_{j=1}^{B}\exp f_\phi(s_i,s_{+,i},\textcolor{brown}{z_j}) - f_\phi(s_i,s_{+,i},\textcolor{teal}{z_i})\right].
\end{align*}
At optimum, the logit recovers the true classifier
\begin{align*}
    f_\phi^{*}(s,s_+,z) = \log \frac{p(z\mid s,s_+)}{p(z)} + c_\phi(s,s_+)
\end{align*}
where $c_\phi(s,s_+)$ denotes the log partition function
\begin{align*}
    c_\phi(s,s_+) \triangleq \log \E_{p(z)} \left[\exp\!\left(f_\phi(s,s_+,z)\right)\right].
\end{align*}

The learned logit gives a variational lower bound on the \gls{mi} $I(Z;S, S_+ )$:
\[
\E_{p(z,s,s_+)}\bigl[f_\phi (s, s_+, z) - c_\phi(s,s_+)\bigr] \leq I(Z;S,S_+),
\]
which becomes tight at the optimum \citep{ma2018noise}.

\paragraph{Adding an Information Bottleneck.} 
High \gls{mi} {\it does not directly imply interesting behavior in continuous settings}. Prior works note that without any type of constraint on the information maximization problem, empowerment is ill-defined in continuous settings and can be infinite \citep{salge2014empowerment}.
In our case, a fully-general skill decoder trained over the transitions can give arbitrarily large empowerment estimates, even for skills that barely change the state.

For the inner skill optimization problem, it is critical to apply an information bottleneck from $S_0, S_+ \rightarrow Z$ when fitting the encoder in the continuous setting. For simplicity, we uniformly add Gaussian noise to the state when training the classifier. This noising forces the classifier to identify skills from perturbations in the state beyond this noise scale. Concretely, we train $f_\phi(s+\epsilon,s_++\epsilon_+,z)$, redrawing the noise $\epsilon,\epsilon_+\overset{\mathrm{iid}}\sim\mathcal N(0,\sigma^2I)$ at each encoder update. \methodName does not noise states during policy and value optimization and solely applies the bottleneck during classifier training. Noise scales for each environment are given in \cref{app:experimental_setup}. Craftax-Classic results use no noise.

\subsection[Practical Considertions: Data Collection \& Conditional MI]{Data Collection \& Conditional MI}
\label{app:data_collection}

During data collection, \methodName switches from the potential empowerment at a switch state $s_0$ to the effective empowerment policy. The purpose of this switch is to ensure that the method learns skills precisely at reachable locations. During rollout, we take the skill distribution $p(z \mid s_0) = p(z)$ to be uniform. Then, classifying skill $z$ from the switch state $s_0$ and the reached future states gives, in expectation, a variational lower bound on the empowerment (\cref{app:theory}). In practice, however, training an encoder over this data can lead to the classifier severe overfitting on the context $s_0$, particularly in continuous settings. 

We can address this overfitting in two ways. The first is to instead learn a joint MI encoder, and fit an encoder to a sampled state in a trajectory (not necessarily the switch point) and a discounted future state, which creates a larger diversity of samples; we call this approach \methodName with a Joint MI classifier. The second is to generate additional trajectories from diverse switch points across the collected trajectories, preserving the conditional \gls{mi} definition of empowerment; we call this approach \methodName with a Conditional MI (CMI) classifier. \emph{The difference between the conditional MI versus the joint MI classifier variants of \methodName happens at the level of data collection and encoder training.}

\paragraph{Why the CMI variant estimates conditional MI}
If skills are sampled uniformly-at-random from the initial state $s_0$, the \gls{mi} $I(Z;S_0)$ equals $0$. Thus, a simple way to optimize the conditional \gls{mi} during the inner \gls{misl} optimization is to roll out skills sampled from the uniform marginal $p(z)$ and train the classifier as usual. Concretely, we copy a switch state $s_0$ then roll out fresh skill(s) from this switch state (i.e. without resets). The classifier only trains on these ``branching'' rollouts. Then, the distributions $p(z\mid s_0)$ and $p(z)$ are equivalent, and the joint \gls{mi} and conditional \gls{mi} are equivalent:
\begin{align*}
I(Z;S_0,S_+)&=I(Z;S_0)+I(Z;S_+\mid S_0)
=I(Z;S_+\mid S_0).
\end{align*}
An alternative approach is to train an estimator of $p(z \mid s)$; we did not pursue this approach to keep the method simple.

\subsection{Training empowerment value functions and policies}
\label{app:value_functions}

\subsubsection{Inner Optimization: Skill Learning/Effective Empowerment}
We consider continuous skill spaces for all of our learned skills.

\paragraph{Markovian Skill Policies \& Critics in Continuous Settings.}
The skill policies and value functions for the inner optimization problem condition on the start state, the current state, and a skill (\cref{sec:learn_skills}). As a result, the skill policy is a function of $(s, s_0, z)$ and the skill critic is a function of $(s, s_0, a, z)$. Meanwhile, the potential policy simply uses the pointwise estimate of the potential empowerment as a state-wise reward, so the potential policy and critics take as input $s$ and $(s,a)$, respectively.
Each \acrshort{sac} critic consists of two independent multilayer perceptrons, each with two 512-unit hidden layers, GELU activations, layer normalization, and a scalar Q-value output. 

\paragraph{Nonmarkovian Skill Policies \& Critics in Crafter.} In the case of a nonmarkovian policy as used in the Crafter environment, which is partially-observed, the policies and \acrshort{pqn}~\citep{gallici2025simplifying} critics take in an additional recurrent state. Following \citep{gallici2025simplifying}, both of the potential empowerment and skill learning Q-networks use a 512-unit fully connected layer with layer normalization and ReLU, followed by a 512-unit LSTM and a linear output giving one Q-value per action. The LSTM receives the full symbolic observaitons, the previous action, and the encoded inputs. The hidden and cell states carry information from earlier steps. 

In addition the recurrent state, the skill Q-network also receives the observation at the skill switch/initiation of the skill and the skill latent $z$. The recurrent state resets at the start of each episode and at the skill switch, so each skill starts from a fresh memory. An alternative option is to carry through the recurrent state through the skill switch, where empowerment is now defined at the recurrent state rather than the observation. Carrying the recurrent state through the skill switch did not substantially change the results.

\paragraph{Skill Rewards.} As mentioned in \cref{sec:learn_skills}, setting the reward to the normalized classifier score provides a lower bound on the MI in expectation, which we show later in \cref{prop:ba_bound_tight}. In practice, for the inner optimization problem, MISL methods like CSF~\citep{zheng2025misl} utilize the \emph{unnormalized} classifier as the reward, because the latter partition term has a skill-agnostic, anti-exploration effect. In our case, while we desire the exploration effects of the unnormalized reward, for the goal of potential empowerment maximization, we need the inner optimization to actually lower bound the pointwise potential empowerment. Thus, at the start of training, we begin with the unnormalized reward to get exploration ($\lambda=0$), then gradually turn on the anti-exploration term ($\lambda=1$) so the method estimates pointwise empowerment by the end of training:
The effective-policy reward is
\[
r_{\rm eff}(s,a,s';s_0,z)
=f_\phi(s_0,s',\textcolor{teal}{z})-\lambda\log\frac{1}{K}\Bigl(e^{f_\phi(s_0,s',\textcolor{teal}{z})}+\sum_{j=1}^{K-1}e^{f_\phi(s_0,s',\textcolor{brown}{z_j})}\Bigr),
\qquad \textcolor{brown}{z_j}\overset{\mathrm{iid}}{\sim}\mathrm{Unif}(\mathbb S^{d-1}),
\]
with $K = 1024$, where the negative skills $\textcolor{brown}{z_j}$ are fresh draws from the uniform skill prior $p(z)$.
Specifically, \methodName sets $\lambda=0$ for the first 30\% of training, increases it linearly to 1 over the next 50\%, and keeps it at 1 for the final 20\%. In experiments, we see that this annealing does not significantly change training in continuous environments. 
In Craftax-Classic, annealing gives the largest empowerment estimate (\cref{fig:ablation_craftax_lambda}). Meanwhile, using the normalized reward throughout training ($\lambda=1$) learns far fewer achievements. Using the unnormalized reward $\lambda = 0$ leads to the highest achievement count ($>9$ achievements), but the value functions are no longer estimates of a lower bound on empowerment by the end of training and are unnormalized objects.

\label{app:her}
\paragraph{Increasing sample efficiency with HER.}
The effective empowerment value functions and policies are heavily-parameterized and depend on the current state, start-state, and the skill. To improve sample efficiency and collect more training examples from the same set of the data, the base \methodName employs Hindsight Experience Relabeling (\acrshort{her}) for the start-state and skill during value function and policy training: while, during rollout, there is a true start-state, \methodName relabels the start-state by sampling geometrically backwards-in-time, then greedily relabels the executed skill. For example, say that the agent sampled a skill from $p(z)$ at start-state $s_0$, and the entire trajectory enters a replay buffer. During policy and value learning using a particular transition $(s_0, s,a,s'z)$, \methodName resamples an effective $s_0'$ backwards-in-time then sets $z \rightarrow z'$, where $z' = \argmax_{z \in \cZ} G(s_0', s,a,z)$ and $G(s_0', s,a,z)$ is a Monte-Carlo estimate of the cost-to-go. While we acknowledge and observe that this relabeling introduces bias and leads to some skill collapse, this relabeling helps scale the method.

\subsubsection{Outer Optimization: Potential Policy and Critic Optimization}

In the ideal pointwise setting, the potential empowerment critic simply takes the Monte-Carlo estimate of the expected skill critic over $8$ uniformly-sampled skills as the reward:
\[
r_{\rm pot}=\frac18\sum_{k=1}^{8} Q(s,a_k;s,z_k),
\quad z_k\sim p,\quad a_k\sim\pi_{\rm eff}(\cdot\mid s,z_k;s).
\]
Under the stated assumptions, in expectation, this estimate is a lower bound on the pointwise potential empowerment (see \cref{eq:pot_reward}).

\paragraph{Separating Actor and Critic Buffers from the Classifier Buffer.} All the critics used in this method can theoretically work in off-policy settings. However, the skill classifier used in the inner optimization cannot train on an extensively off-policy replay buffer to improve the \emph{current} policy's effective empowerment. \citet{levy2025representation} discuss the importance of training a classifier on the current policy. Thus, we maintain separate buffers for actor-critic training and classifier training, where the classifier buffer contains only recent experience and is generally much smaller than that of the actor-critic buffers. Exact parameters are in \cref{app:experimental_setup}.

\subsection{Empowerment maximization pseudocode}
\label{app:pseudocode}

\begin{algorithm}[htb]
\caption{\methodName}
\label{alg:misl}
\begin{algorithmic}[1]
\Require Skill prior $p(z)$, episode length $T$, switch time $M$, iterations $N$
\State Initialize potential policy $\pip$, skill policy $\pie$, their critics, classifier $f_\phi$, and buffers.
\For{$n=1,\ldots,N$}
    \State \textbf{Collect:} Run $\pip$ for $M$ steps from reset; record the switch state $s_0$ (\cref{sec:practical}).
    \State Draw $z\sim p(z)$ and run $\pie(\cdot\mid s,z;s_0)$ for $T-M$ steps, holding $z,s_0$ fixed.
    \State Store the two-stage rollout for policy and critic updates.
    \State \textbf{Fit classifier:} Train $f_\phi$ with InfoNCE on states and discounted future states
    \Statex \hspace{\algorithmicindent}from skill execution (Joint MI), or branches with fresh skills $z\sim p(z)$
    \Statex \hspace{\algorithmicindent}from copied starting states (CMI; \cref{app:data_collection}).
    \State \textbf{HER (SAC only):} Replace the anchor and skill of replayed skill transitions in hindsight (\cref{app:her}).
    \State \textbf{Inner optimization:} Update $\pie$ and its critics using $r_{\rm eff}$ with the current normalization weight $\lambda(n)$ (\cref{sec:learn_skills,app:value_functions}).
    \State \textbf{Estimate empowerment:} Average skill values over eight draws $z\sim p(z)$ (\cref{eq:pot_reward}).
    \State \textbf{Outer optimization:} Use this estimate as $r_{\rm pot}$ to update $\pip$ and its critic (\cref{sec:misl}).
    \Statex \hspace{\algorithmicindent}SAC uses replay; PQN uses the two-stage rollout; CMI branches train only $f_\phi$.
\EndFor
\State \Return $\pie$, $\pip$
\end{algorithmic}
\end{algorithm}
\FloatBarrier

\section[Experimental setup and evaluation]{Experimental setup and evaluation}
\label{app:experimental_setup}

\paragraph{Continuous Env Evaluation.} Like in the data collection during training, the evaluation rollouts run the potential empowerment policies for the first $100$ steps. Then, the rollout draws a skill $z \sim p(\cdot)$ and executes the effective empowerment policy until termination or truncation for the next $300$ steps. Evaluations average over $128$ episodes in $128$ new, sampled environments.

The potential empowerment target distance is the xy distance \emph{at the switch point}, which is the 100-step potential empowerment phase of the rollout. In Ant, this is the distance from the Ant to the center of the arena; In AntSoccer, this is the distance away from the soccer ball; In BuilderBench, this is the distance to any of the blocks. Meanwhile, the skill displacements in \cref{fig:ablations} measure the displacement from the switch point to the end of the skill rollout (next $300$ steps). In Ant, this is the displacement of the ant; In AntSoccer, this is the displacement of the ball; In BuilderBench, this is the displacement of the blocks.

\paragraph{Craftax Classic Evaluation.} We use the externally-defined achievement count for evaluation in Craftax-Classic, measuring the total achievements accrued over the course of the two-phase rollout in $128$ randomized evaluation environments.

Curves show the mean $\pm$ standard error across training seeds: five for \methodName and LatentGCE, and three for VIM and each ablation.
The learned empowerment metrics are plotted after per-seed EMA ($\alpha=0.05$) before averaging. The achievement metrics are unsmoothed.

\subsection{Continuous Experiment Details}
\label{app:cont_details}
In the figures, Ant denotes the full-state classifier (42 features), while AntSoccer denotes the classifier using only ball $x,y$; BuilderBench uses the three blocks' positions. Visitation counts the cumulative number of distinct planar cells seen during training. In Ant and AntSoccer, the cells are 1\,m$\times$1\,m xy tiles of either the ant (Ant) or ball (AntSoccer) location, with a total of $576$ cells. Note that \methodName can quickly get complete or close-to-complete coverage in both Ant and AntSoccer. In BuilderBench, the cells are 5\,cm$\times$5\,cm xy tiles pooled across the three blocks, totaling $192$ cells. For \methodName and ablations, we record and plot the visitation metrics for the skill phase of the rollouts. Because the baseline LatentGCE and VIM methods do not have a skill phase, we instead plot the visitation metrics over the entirety of the training rollouts.

\subsubsection{Continuous hyperparameters}
\label{app:exp_ant}
\label{app:exp_ball}
\label{app:exp_cube}

Hyperparameters are in \cref{tab:ab}. Training curves are in \cref{fig:training_curves} and learned \gls{mi} curves are in \cref{fig:continuous_empowerment}.

Both Ant and AntSoccer use the same arena, but use different encoder observations when training the classifier.  In Ant, the ant starts in a top-right start cell with a uniform, random offset in $[-1,1]^2$ m. The ball sits 1.5\,m past the center. In AntSoccer, the ant starts $80\%$ of the way along a line from the ant start to the center of the arena, where the ball starts at the center. BuilderBench sets the three cubes in a deterministic position in the corner. The gripper initialization is also deterministic at the center-edge of the play area, far from the blocks.

\begingroup
\small
\begin{longtable}{@{}>{\raggedright\arraybackslash}p{0.34\linewidth}*{3}{>{\raggedright\arraybackslash}p{\dimexpr0.22\linewidth-2\tabcolsep\relax}}@{}}
    \caption{Hyperparameters for \methodName with joint MI in continuous environments.}\label{tab:ab}\label{tab:ball}\label{tab:cube}\\
    \toprule
    Parameter & Ant & AntSoccer & BuilderBench \\
    \midrule
    \endfirsthead
    \multicolumn{4}{l}{\tablename\ \thetable\ (continued)}\\
    \toprule
    Parameter & Ant & AntSoccer & BuilderBench \\
    \midrule
    \endhead
    \midrule
    \multicolumn{4}{r}{Continued on next page}\\
    \endfoot
    \bottomrule
    \endlastfoot
    Encoder observations & Full state (42) & Ball $x,y$ (2) & Block positions (9) \\
    Observation noise std & 0.05 & 0.05 & 0.01 \\
    Skill dimension & 2 & 2 & 9 \\
    Online env steps per seed & 500M & 500M & 100M \\
    Replay capacity (transitions) & 10M & 10M & 6M \\
    Batch size & 1,024 & 1,024 & 256 \\
    Target entropy & 0 & 0 & $-12.5$ \\
    $\lambda$ annealing (env steps) & 150M--400M & 150M--400M & 30M--80M \\
    \midrule
    \multicolumn{4}{l}{\textbf{Shared settings}} \\
    Learner & \multicolumn{3}{l}{SAC} \\
    Training seeds & \multicolumn{3}{l}{5} \\
    Skill prior & \multicolumn{3}{l}{Uniform unit sphere} \\
    Potential / skill phase & \multicolumn{3}{l}{100 / 300 steps} \\
    Reward partition weight $\lambda$ & \multicolumn{3}{p{\dimexpr0.66\linewidth-2\tabcolsep\relax}@{}}{0 before the interval, linear increase to 1 from $30-80\%$ of training, 1 thereafter} \\
    Encoder buffer & \multicolumn{3}{l}{3,334 episodes} \\
    Learning rate & \multicolumn{3}{l}{$3\times10^{-4}$} \\
    SAC / encoder samples per collected transition & \multicolumn{3}{l}{4 / 1} \\
    Critic target update rate & \multicolumn{3}{l}{0.02} \\
    Policy / future-sampling discount & \multicolumn{3}{l}{0.99 / 0.99} \\
    Contrastive temperature & \multicolumn{3}{l}{0.1} \\
    Entropy temperature & \multicolumn{3}{l}{Learned, initialized to 1} \\
    Actor / critic / encoder hidden layers & \multicolumn{3}{l}{[512, 512]} \\
    Layer norm & \multicolumn{3}{l}{Critic and encoder only} \\
    Potential-value skill samples & \multicolumn{3}{l}{8} \\
\end{longtable}
\endgroup

\subsection{Discrete Experiment Details}
\label{app:disc_details}

\subsubsection{Craftax-Classic}
\label{app:exp_craftax}

Hyperparameters are in \cref{tab:craftax}. \cref{fig:craftax_tree} shows one two-phase episode. We use  PQN~\citep{gallici2025simplifying} rather than SAC to utilize the GPU acceleration in the Craftax environment.

\begin{figure}[htb]
    \centering
    \includegraphics[width=0.5\linewidth]{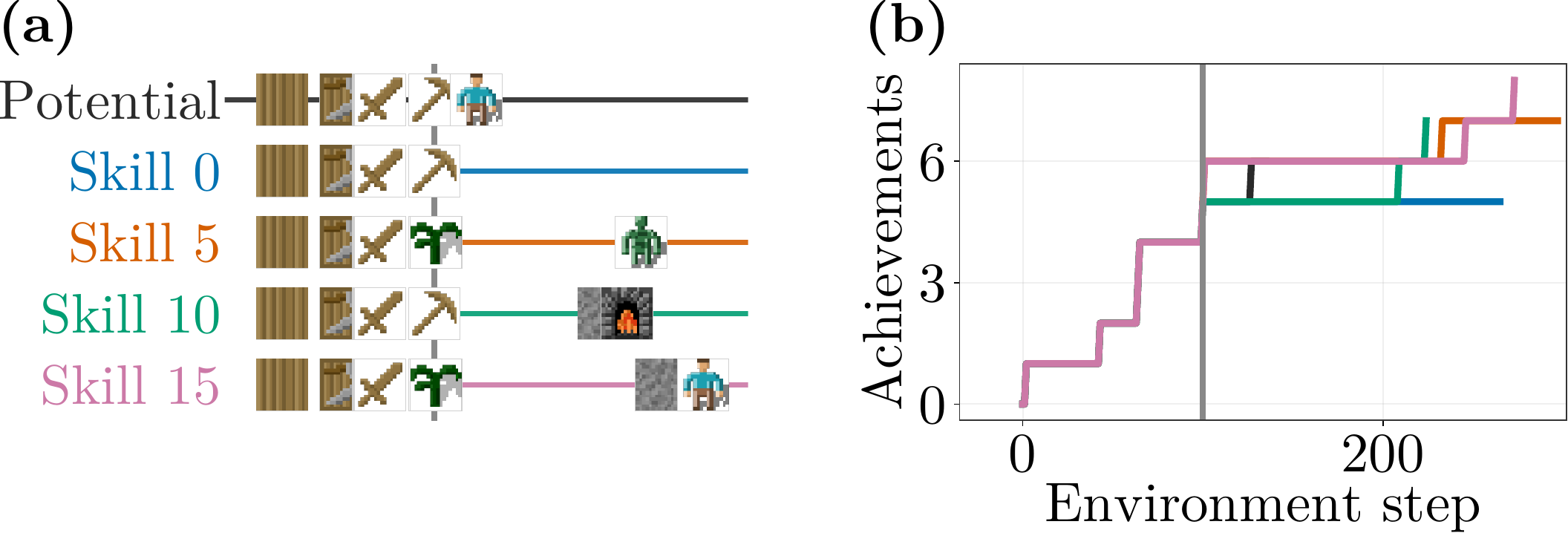}
    \caption{\textbf{Achievements in a single Craftax-Classic two-phase episode.} Icons indicate when achievements first occur. (a) The black line shows the achievements of the shared potential phase up to the skill switch (vertical gray line). In this rollout, these achievements include collecting wood, building a crafting a table, then crafting a wooden sword. The colorful skill rows show the achievements rolled out from the switch point sharing the same starting state, including sleeping/waking up, building a pickaxe, planting a tree, collecting stone, building a furnace, and defeating an enemy. (b) The lines indicate the cumulative achievement counts over the rollout, achieved without the use of any extrinsic reward. One seed.}
    \label{fig:craftax_tree}
\end{figure}

\begingroup
\small
\begin{longtable}{@{}>{\raggedright\arraybackslash}p{0.46\linewidth}>{\raggedright\arraybackslash}p{\dimexpr0.54\linewidth-2\tabcolsep\relax}@{}}
    \caption{Hyperparameters in Craftax-Classic (SEE, CMI (Ours); quantitative curves).}\label{tab:craftax}\\
    \toprule
    \textbf{Parameters} & \textbf{Value} \\
    \midrule
    \endfirsthead
    \multicolumn{2}{l}{\tablename\ \thetable\ (continued)}\\
    \toprule
    \textbf{Parameters} & \textbf{Value} \\
    \midrule
    \endhead
    \midrule
    \multicolumn{2}{r}{Continued on next page}\\
    \endfoot
    \bottomrule
    \endlastfoot
        Encoder observations & inventory (12 features) \\
        Encoder data & CMI branches from copied states \\
        Observation noise std (classifier bottleneck) & 0 \\
        Future-sampling discount & $0.95$ \\
        Logit temperature & 0.1 (fixed) \\
        Encoder learning rate & 1e-3 \\
        Encoder batch size & 4{,}096 \\
        Encoder hidden dims & [512, 512] \\
         Encoder Buffer & 10{,}000{,}000 branch transitions \\
        \midrule
        Skill dimension & 16 \\
        Skill prior & uniform unit sphere \\
        Potential / skill phase & 100 / 900 steps \\
        Training seeds & 5 \\
        Encoder branch horizon & 900 \\
        Extra simulator steps per live step & 0.9 from encoder branches, excluded from the live-step budget \\
        \midrule
        Online env steps per seed (live) & 2{,}000{,}000{,}000 \\
        \midrule
        Potential-value skill samples & 8 \\
        Skill / anchor relabeling & Disabled \\
        Reward partition weight $\lambda$ & 0 before 600M live steps, linear increase to 1 at 1.6B, 1 thereafter \\
        \midrule
        Learning rate & 3e-4 (linear decay to 0) \\
        Minibatches per rollout & 4, each with 256 sequences of 128 steps \\
        Update epochs & 4 \\
        Encoder updates per rollout batch & 4 \\
        Exploration & $\epsilon$-greedy: $1\!\to\!0.005$ over 100M live env steps \\
        Discount $\gamma$ & 0.95 \\
         PQN $\lambda$ & 0.5 \\
        Max gradient norm & 0.5 \\
        Q-network input normalization & Batch renormalization \\
\end{longtable}
\endgroup

\section{Baselines}

\label{app:baselines}

Unlike our method, both \gls{vim}~\citep{mohamed2015variational} and \gls{latentgce}~\citep{zhao2021efficient} consider open-loop definitions of empowerment, where $A_{0}^{H-1}$ is shorthand for the action sequence $(a_0, \dots, a_{H-1})$:
\begin{equation}
    \cE_{\mathrm{VIM}}(s_0)
    \triangleq
    \max_{\pi \in \Pi_\cZ}
   I^\pi(A_{0}^{H-1};S_{H} \mid S_0=s_0).
\end{equation}

\paragraph{VIM.} \gls{vim} utilizes a variational lower bound on the empowerment with a learned decoder $q_\theta (a_1^H\mid s_0,s_H)$, a generative model for the action sequence. VIM fits this decoder using a standard cross-entropy loss. Then, by definition, learning a source distribution $\omega(a_0^{H-1} \mid s_0)$ that maximizes this variational lower bound gives a lower bound on the true open-loop potential empowerment. We follow the autoregressive source and decoder in \citet[Sec.~4.2]{mohamed2015variational}, using tanh-Gaussians for our continuous actions.

Fixing the estimated decoder $q_\theta (a_1^H\mid s_0,s_H)$, at optimum, the source distribution $\omega(a_0^{H-1}\mid s_0)$ takes the form
\[
\omega^\star(a_0^{H-1}\mid s_0) \propto \exp(\EE_{p(s_H \mid s_0, a_{0}^{H-1})} [\log q_\theta (a_{0}^{H-1} \mid s_0, s_H)]).
\]
However, computing the partition function/normalizer is generally intractable. Instead, VIM regresses a function $f_\phi(s_0, a_{0}^{H-1}) \triangleq \log h_\phi (a_{0}^{H-1} \mid s) + \psi_\phi (s_0)$ onto the energy $\EE_{p(s_H \mid s_0, a_{0}^{H-1})} [\log q_\theta (a_{0}^{H-1} \mid s_0, s_H)]$, where $h_\phi$ is normalized. At optimum, the function $\psi_\phi (s_0)$ converges to an approximation of the open-loop potential empowerment by the standard variational Gibbs identity. We refer the reader to \citet{mohamed2015variational} for more details.

\paragraph{LatentGCE.} The next baseline we consider is \gls{latentgce}~\citep{zhao2021efficient}. Unlike \gls{vim}, \gls{latentgce} learns a generative, forward-in-time model to estimate empowerment. 
\gls{latentgce} learns a latent model of the environment that is a linear Gaussian channel, which automatically gives a closed-form expression for the open-loop potential empowerment (under power constraints).
 This approach sidesteps the challenges of optimizing a source distribution by directly learning a generative model with a linear Gaussian form.

\subsection{Validating Empowerment Fields in Simple Continuous Benchmarks}

We compare the learned empowerment fields in two simple, didactic continuous environments of a pendulum and ball-in-box discussed in LatentGCE. The learned empowerment fields are in \cref{fig:baseline_fields}. Our JAX implementations of LatentGCE (based on the original LatentGCE codebase) and VIM qualitatively reproduce the behavior reported in \cite{zhao2021efficient}, where LatentGCE succeeds to find the high empowerment region while VIM does not.

\begin{figure}[htb]
    \centering
    \includegraphics[width=0.75\linewidth]{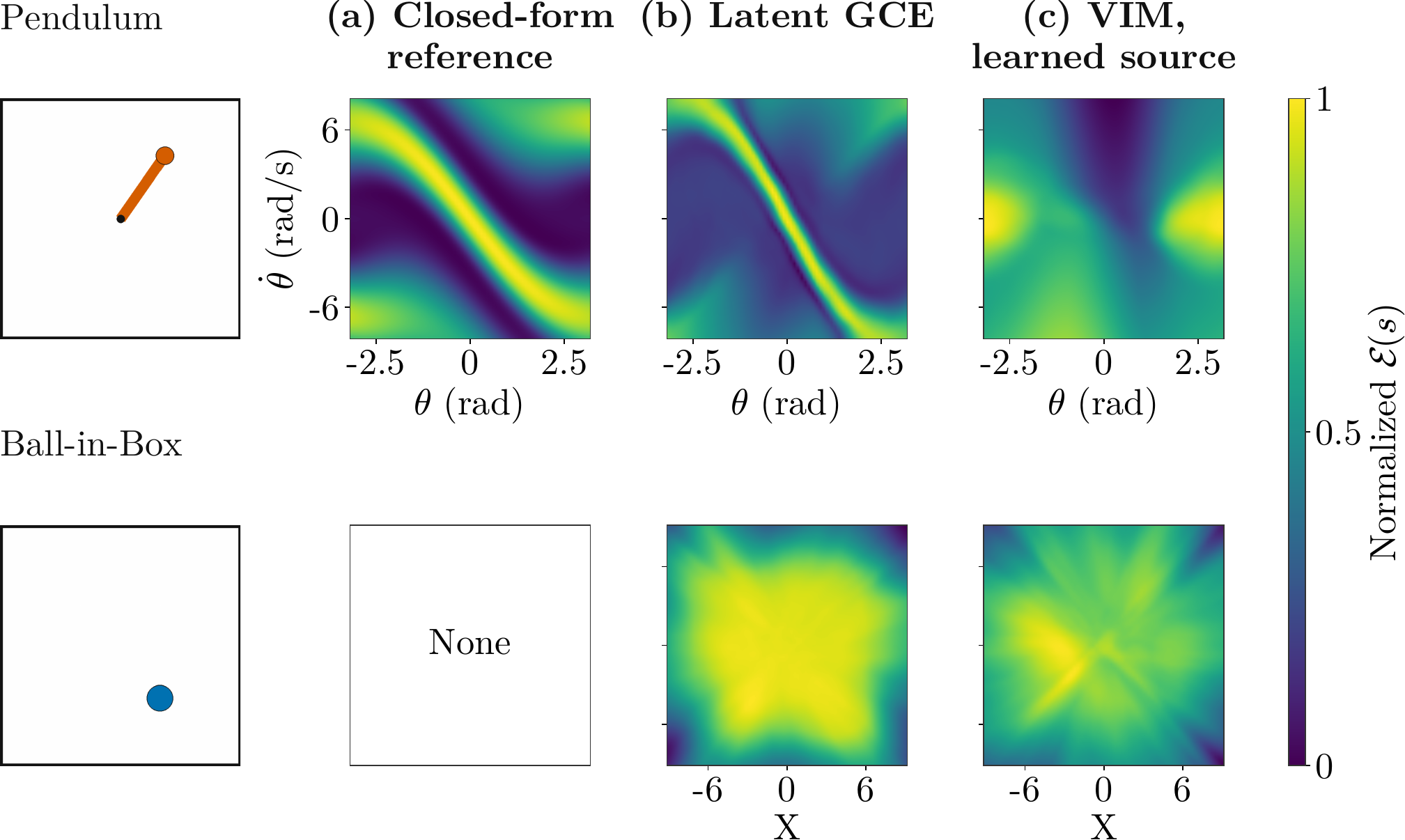}
    \caption{\textbf{Empowerment fields on low-dimensional benchmarks.} For the canonical environments of a pendulum and a ball in a box with hard walls, various empowerment maximization methods predict (normalized) empowerment fields over state space. (a) A linearized reference reveals that the potential empowerment is maximized when the pendulum points straight up. For the Ball in Box, the potential empowerment is maximal at the center of the box. Thus, a potential empowerment policy over a long effective horizon should seek a straight-up state in the pendulum and the central state for Ball in Box. (b) The empowerment field predicted by LatentGCE method correctly identifies these potential empowerment states. (c) VIM does not recover these high empowerment regions. One seed per learned field.}
    \label{fig:baseline_fields}
\end{figure}

\subsection[Baseline Implementations in Scaled Environments]{Implementation in the main evaluation environments}
\label{app:baseline_implementation}

We then apply VIM and LatentGCE to Ant, AntSoccer, and BuilderBench following the referenced manuscript methodology. Unfortunately, neither VIM nor LatentGCE provide full public codes of their methods. In particular, VIM has no publicised codebase and the public LatentGCE codebase does not implement the version of their method that learns latents. As such, we adapted the methodology discussed in the LatentGCE paper (originally applied to pixel-based tasks) in addition to their partial codebase to our scaled continuous settings \citep{zhao2021efficient}. We implemented VIM following the algorithm block included in the paper and include these implementations within our public codebase \citep{mohamed2015variational}. Key hyperparameters are included in \cref{tab:vim_hyperparameters,tab:latentgce_hyperparameters}.

Both methods estimate the empowerment over the same degrees of freedom as \methodName: full state in Ant, ball in AntSoccer, and block positions in BuilderBench. In LatentGCE, we replace the convolutional encoders originally designed for the pixel-based settings with standard MLPs. Both methods optimize the potential empowerment policies using \acrshort{ppo}. Since both methods estimate open-loop empowerment, we consider empowerment over open-loop action sequences of length $H=30$ for VIM and $H=100$ for LatentGCE. We chose a shorter sequence length for VIM due to the observed instability of the algorithm. We chose $H=100$ for LatentGCE to match the effective trajectory length scale for the rollouts elsewhere in the paper, corresponding to a disount factor $\gamma =0.99$. \Cref{tab:vim_hyperparameters,tab:latentgce_hyperparameters} list the key hyperparameters, with full implementation provided in the linked code. In general, we find that VIM exhibits unstable training. In VIM, defining an energy for the source action sequence distribution with the learned decoders led to unstable source updates, which then negatively impacted future data collection. 
In earlier attempts at running LatentGCE, exponentiating the empowerment estimate produced large rewards~\citep{zhao2021efficient}. Thus, we instead set the reward as the empowerment, following the released LatentGCE pseudocode and the methodology in the rest of this paper.

\begin{table}[htbp]
\centering
\small
\caption{VIM hyperparameters for Ant, AntSoccer, and BuilderBench. Each environment uses 3 training seeds.}\label{tab:vim_hyperparameters}
\begin{tabular}{@{}>{\raggedright\arraybackslash}p{0.46\linewidth}>{\raggedright\arraybackslash}p{\dimexpr0.54\linewidth-2\tabcolsep\relax}@{}}
\toprule
Parameter & Value \\
\midrule
Training budget (Ant / AntSoccer / BuilderBench) & 500M / 500M / 100M \\
Episode length / probe horizon $H$ & 400 / 30 steps \\
Parallel environments / rollout length & 256 / 80 steps \\
Probe starts per update & 256, one uniformly sampled rollout time per environment \\
\midrule
\multicolumn{2}{l}{\textbf{PPO controller}} \\
Actor / critic hidden layers & [256, 256], tanh \\
Optimizer / learning rate & Adam / $3\times10^{-4}$, linear decay \\
Discount $\gamma$ / GAE $\lambda$ & 0.99 / 0.95 \\
Update epochs / minibatches & 4 / 4 \\
PPO policy and value clipping / max. gradient norm & 0.2 / 0.5 \\
Entropy / value loss coefficients & 0.01 / 0.5 \\
Action distribution & Tanh-Gaussian, actions in $[-1,1]^{d_a}$ \\
Observation normalization & Running mean and variance \\
\midrule
\multicolumn{2}{l}{\textbf{Empowerment estimator}} \\
Estimator replay capacity / minibatch & 10,000 tuples / 1,024 \\
Estimator passes over replay per update & 10 \\
Learning rate & $3\times10^{-4}$ \\
Hidden layers & Source/decoder [512, 512, 512] \\
Additional hidden layers & Empowerment head [256, 256, 256] \\
Activation & ReLU \\
VIM energy coefficient & 1 \\
Controller reward & Estimated empowerment, clipped below at zero (due to instability) \\
\bottomrule
\end{tabular}
\end{table}

\begin{table}[htbp]
\centering
\small
\caption{LatentGCE hyperparameters for Ant, AntSoccer, and BuilderBench. Each environment uses 5 training seeds.}\label{tab:latentgce_hyperparameters}
\begin{tabular}{@{}>{\raggedright\arraybackslash}p{0.46\linewidth}>{\raggedright\arraybackslash}p{\dimexpr0.54\linewidth-2\tabcolsep\relax}@{}}
\toprule
Parameter & Value \\
\midrule
Training budget (Ant / AntSoccer / BuilderBench) & 500M / 500M / 100M \\
Episode length / probe horizon $H$ & 400 / 100 steps \\
Parallel environments / rollout length & 256 / 80 steps \\
Probe starts per update & 256, one uniformly sampled rollout time per environment \\
\midrule
\multicolumn{2}{l}{\textbf{PPO controller}} \\
Actor / critic hidden layers & [256, 256], tanh \\
Optimizer / learning rate & Adam / $3\times10^{-4}$, linear decay \\
Discount $\gamma$ / GAE $\lambda$ & 0.99 / 0.95 \\
Update epochs / minibatches & 4 / 4 \\
PPO policy and value clipping / max. gradient norm & 0.2 / 0.5 \\
Entropy / value loss coefficients & 0.01 / 0.5 \\
Action distribution & Tanh-Gaussian, actions in $[-1,1]^{d_a}$ \\
Observation normalization & Running mean and variance \\
\midrule
\multicolumn{2}{l}{\textbf{Empowerment estimator}} \\
Estimator replay capacity / minibatch & 10,000 tuples / 1,024 \\
Estimator passes over replay per update & 10 \\
Learning rate & $10^{-4}$ \\
Hidden layers & Encoders/channel [512, 512, 512] \\
Additional hidden layers & Reconstruction decoders [512, 512] \\
Activation & ReLU \\
State / action latent dimension & 32 / 32 \\
Channel power budget & 1 \\
Latent prediction / state reconstruction loss weights & 1 / 500 \\
Controller reward & Open-loop empowerment estimate in bits \\
Action reconstruction output & Tanh in $[-1,1]$ \\
Action reconstruction / latent norm loss weights & 10 / 1 \\
\bottomrule
\end{tabular}
\end{table}

\begin{figure}[htbp]
    \centering
    \includegraphics[width=0.9\linewidth]{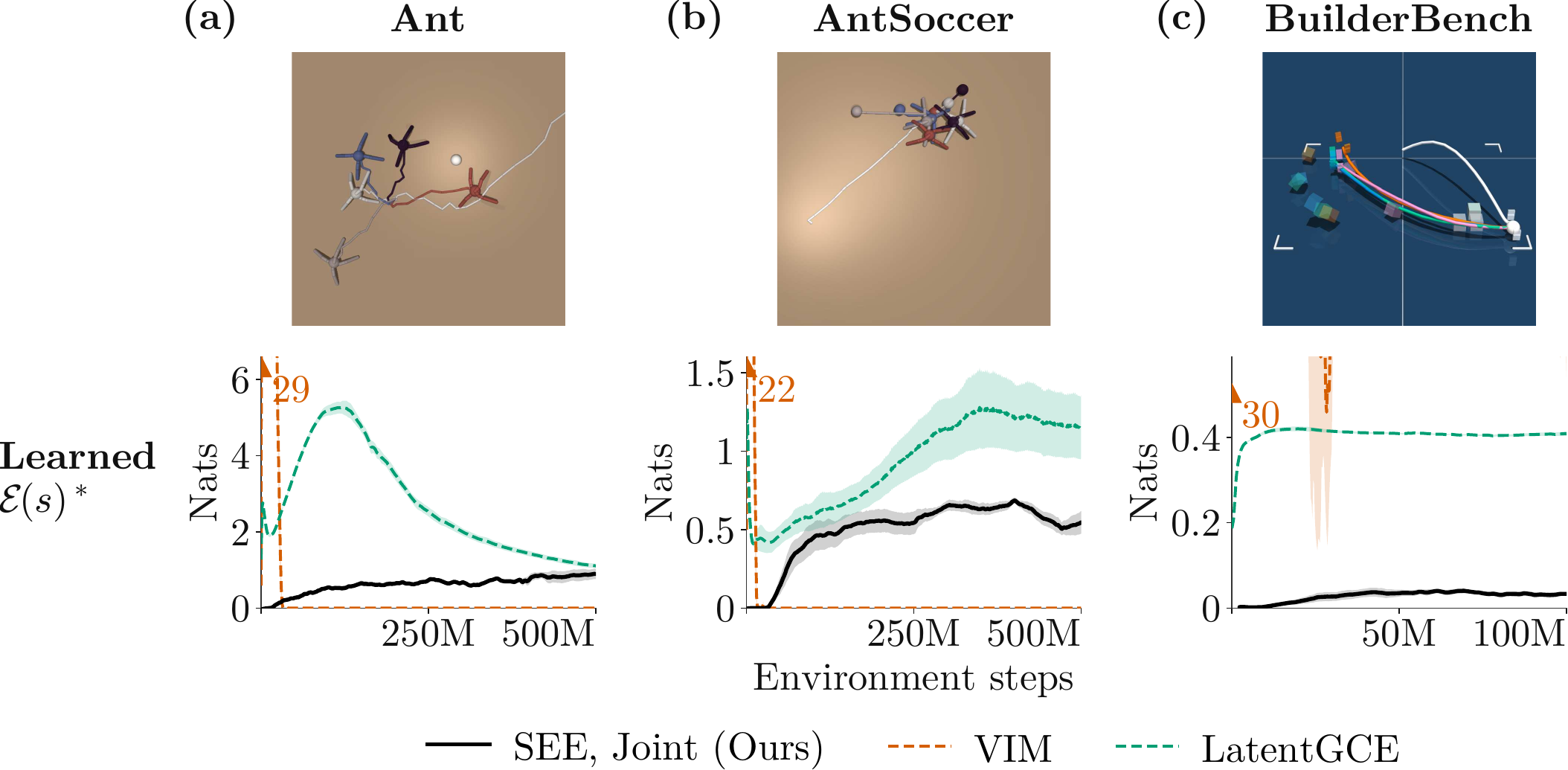}
    \caption{\textbf{Learned empowerment in continuous settings.} \methodName (Black) empowerment values gradually increase over time. Meanwhile, VIM (Orange) empowerment values are unstable and are very large at the start of training (Orange arrows point beyond the displayed y-axis). The LatentGCE (Green) empowerment values do not monotonically increase in time. *As a note, learned estimates use different objectives and regularization, so their magnitudes are not directly comparable across methods. Five seeds for \methodName and LatentGCE and three for VIM, p.m. standard error.}
\label{fig:continuous_empowerment}
\end{figure}

\section[Ablation studies]{Ablation Details}

Here, we include additional ablations and ablation details. The corresponding figures are in \cref{fig:ablations} and \cref{fig:ablations_craftax_structural}. Default \methodName utilizes the joint \gls{mi} classifier in the continuous settings and the CMI classifier in the Crafter setting.

\subsection{Policy switching.}

The first design decision we ablate is policy switching. An alternative is to separately train potential empowerment and effective empowerment policies:
\begin{description}[style=standard, widest={Ablation:}, leftmargin=*, labelindent=0pt, align=left, labelsep=0.6em, nosep]
    \item[\textbf{Ours:}] Policy switching from potential to effective during data collection.
    \item[\textbf{Ablation:}] Training skills and potential policy from the the same starting state, without skill switching. To allocate similar training time to the potential and skill phases as in the base \methodName, we roll out using the potential policy $25\%$ of the time, and the skill phase the other $75\%$.
\end{description}
 Without switching, the method cannot consistently learn meaningful potential empowerment policies in Ant and BuilderBench. The potential policy ends up far away from the center of the arena/ball/blocks in Ant/AntSoccer/BuilderBench and does not accumulate many achievements in Crafter-Classic.

\subsection{Conditional versus full \gls{mi}.}

CMI uses the branching rollout scheme detailed in \cref{app:data_collection}. All other training stays fixed. We observe that using CMI worsens training in the high-dimensional, continuous setting of Ant where it is easy for a properly context-conditioned classifier to overfit. The potential policy does not move much. Meanwhile, using the Joint MI in Craftax-Classic leads to collapsed empowerment values despite high achievement scores. In BuilderBench, CMI and \methodName perform similarly. In AntSoccer, \methodName kicks the ball further in the skill phase, though both methods learn potential policies that approach the ball.

\subsection[Pair-decoder parameterization]{ Pair-decoder parameterization.}

We compare two logits:
\begin{align*}
\text{Difference encoder:}\quad &f_\phi(s,s_+,z)=z^\top[\phi(s_+)-\phi(s)]/\tau,\\
\text{Pair decoder:}\quad &f_\psi(s,s_+,z)=z^\top g_\psi([s,s_+])/\tau.
\end{align*}
We use the same network architecture for both logit parameterizations. All other experiment details are identical to base \methodName. The pair decoder worsens visitation, target distance, and skill displacement in Ant and AntSoccer, despite having higher learned empowerment estimates in Ant. For example, in AntSoccer, the pair decoder visits $280$ cells versus $539$ for \methodName. The learned skills barely move the ball. 
This ablation highlights the importance of regularizing the classifier when doing empowerment maximization in scaled settings.

\begin{figure}[htb]
    \centering
    \includegraphics[width=\linewidth]{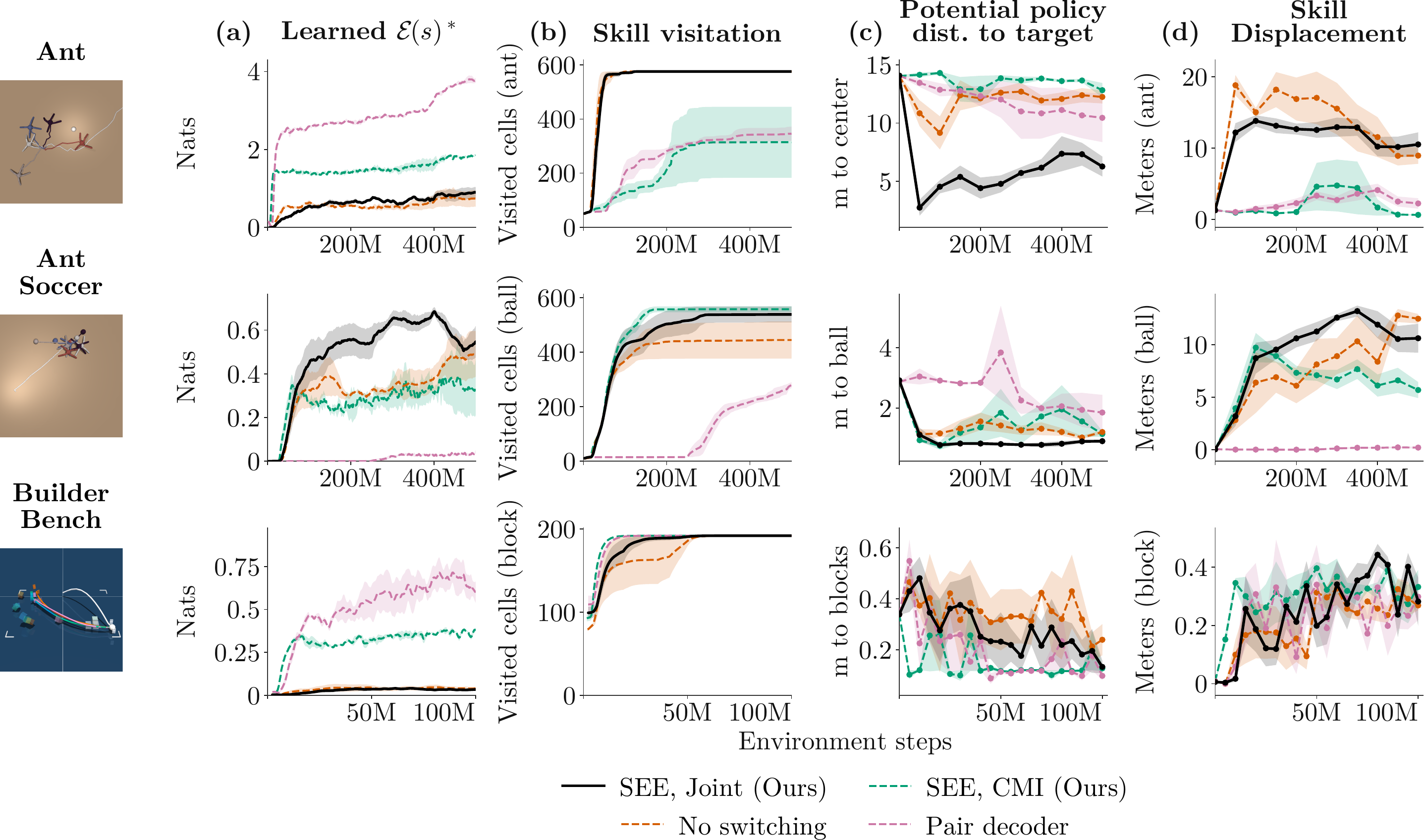}
    \caption{We compare the full method with variants that remove the two-phase rollout (no switching, orange), replace joint \gls{mi} with conditional \gls{mi} (green), or replace the bilinear logit with a general MLP (pair decoder, purple). The solid black curve shows \methodName, and the colored dashed curves show the ablations. (a) For all methods, average empowerment increases over training. $^*$ Learned empowerment estimates use different objectives and regularization, so their magnitudes are not directly comparable across methods. (b) The skill visitation of \methodName quickly saturates for these environments. Meanwhile, the ablations struggle to get visitation in a subset of the different environments. (c) \methodName consistently learns potential policies that reach the center of the arena (Ant), approach the ball (AntSoccer), and reach the blocks (BuilderBench). In BuilderBench, conditional \gls{mi} and the pair decoder potential empowerment rollouts reach the blocks, while no switching ends far away. In Ant, all ablations end far from the center of the arena. In AntSoccer, the ablations end up further from the ball than \methodName. (d) \methodName consistently learns effective policies with large displacements. Five seeds for \methodName; three per ablation.} 
    \label{fig:ablations}
\end{figure}

\begin{figure}[htb]
    \centering
    \includegraphics[width=\linewidth]{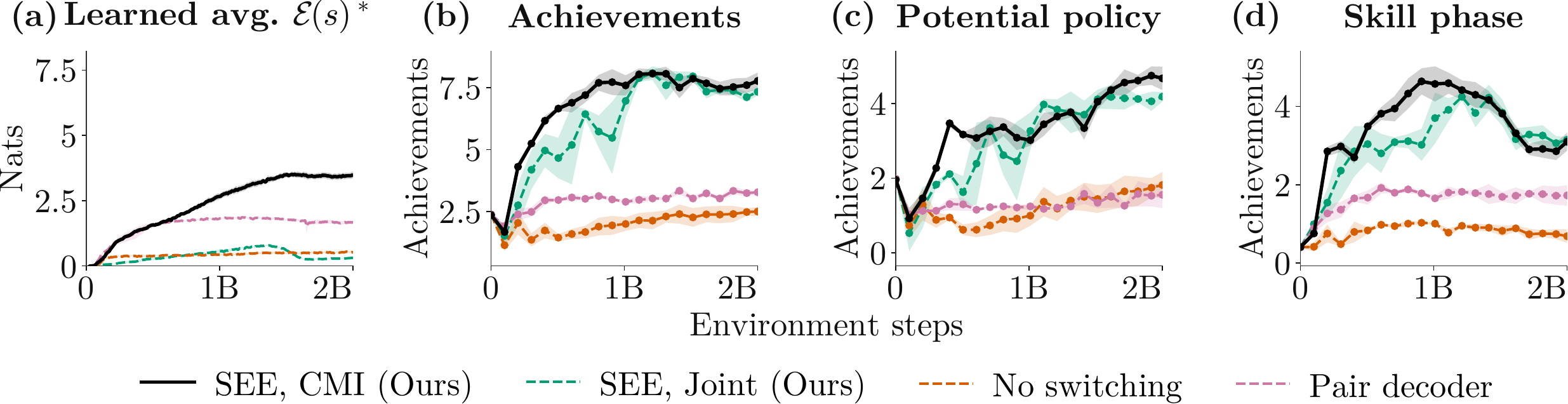}
    \caption{Craftax-Classic ablations (CMI \methodName: 5 seeds; each ablation: 3). The black curve is the \methodName with the CMI classifier (Ours). The ablations are SEE with Joint MI (Ours), no switching (25\% potential-policy episodes during training), and pair decoder. (a) Learned empowerment estimates rise over training for the base method. (b) Total achievements per rollout. (c) Achievements in the first 100 potential-policy actions. (d) New achievements in the subsequent skill phase of up to 900 actions. All methods use 128 fresh worlds and skills per seed and evaluation checkpoint. Curves show mean $\pm$ standard error across training seeds, with per-seed EMA ($\alpha=0.05$) only in (a). Achievement panels are unsmoothed and display every tenth evaluation checkpoint plus the final checkpoint.} \label{fig:ablations_craftax_structural}
\end{figure}

\subsection{Reward normalization}
\label{app:reward_normalization}

We compare the normalization schedule in \cref{app:value_functions} with fixed $\lambda=0$ and $\lambda=1$ (\cref{fig:ablation_rewards,fig:ablation_craftax_lambda}). Overall, we find that the normalization schedule does not significantly change the training metrics in continuous settings (\cref{fig:ablation_rewards}). However, in Craftax-Classic, normalization $(\lambda = 1)$ throughout training significantly hurt the learned average empowerment values (derived from the classifier). In this setting, we find that the normalization schedule is critical to learn value functions that estimate lower bounds of the potential empowerment (\cref{fig:ablation_craftax_lambda}). 

\begin{figure}[htb]
    \centering
    \includegraphics[width=\linewidth]{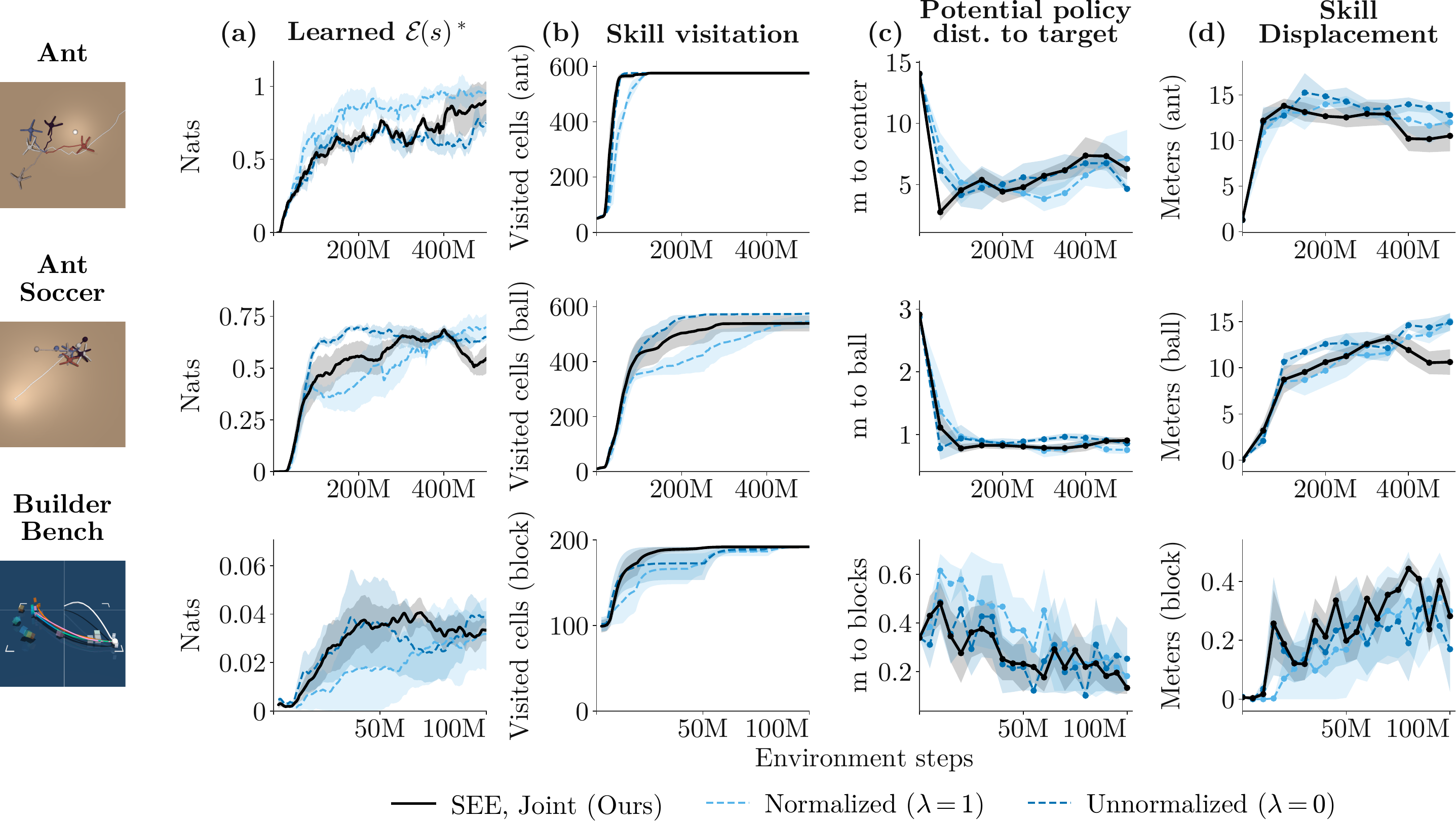}
    \caption{\textbf{Reward normalization in continuous environments.} We compare the annealed schedule with normalized ($\lambda=1$) and unnormalized ($\lambda=0$) rewards. Columns show learned empowerment, training skill visitation, target distance after 100 potential-policy actions, and displacement during the subsequent 300 skill actions. Learned estimates use different objectives and are not directly comparable across settings. Five seeds for \methodName; three per ablation. Curves show mean $\pm$ standard error across seeds. Training curves use per-seed EMA ($\alpha=0.05$); evaluation curves are unsmoothed.}
    \label{fig:ablation_rewards}
\end{figure}

\begin{figure}[htb]
    \centering
    \includegraphics[width=0.5\linewidth]{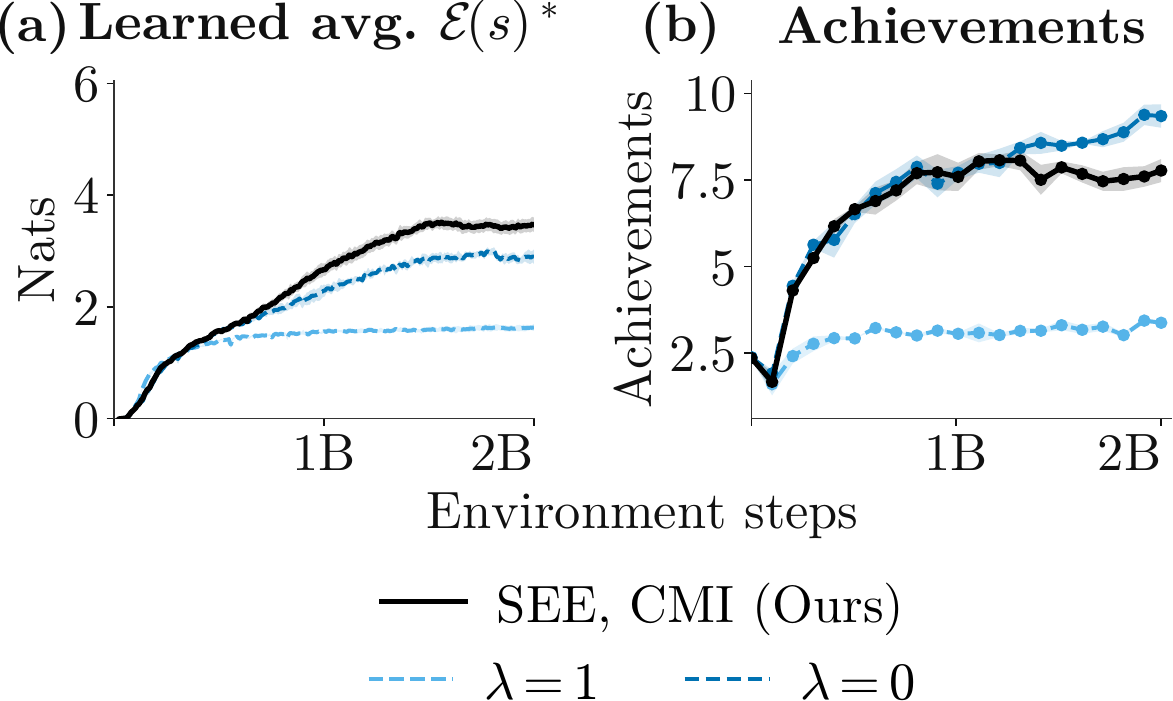}
    \caption{\textbf{Reward normalization in Craftax-Classic.} We compare CMI \methodName with the annealed schedule against fixed $\lambda=0$ and $\lambda=1$. Five seeds for \methodName; three per ablation. Curves show mean $\pm$ standard error across seeds. The learned-estimate panel uses per-seed EMA ($\alpha=0.05$). Achievements use 128 fresh worlds and skills per seed and checkpoint, with 100 potential-policy actions followed by up to 900 skill actions. Achievement evaluations are unsmoothed; every tenth checkpoint and the final checkpoint are shown.}
    \label{fig:ablation_craftax_lambda}
\end{figure}

\clearpage
\section[Additional qualitative examples]{Additional qualitative examples}
\label{app:qualitative}

We include additional examples of emergent environment modification across all seeds.

\begin{figure}[htb]
    \centering
    \includegraphics[width=0.5\linewidth]{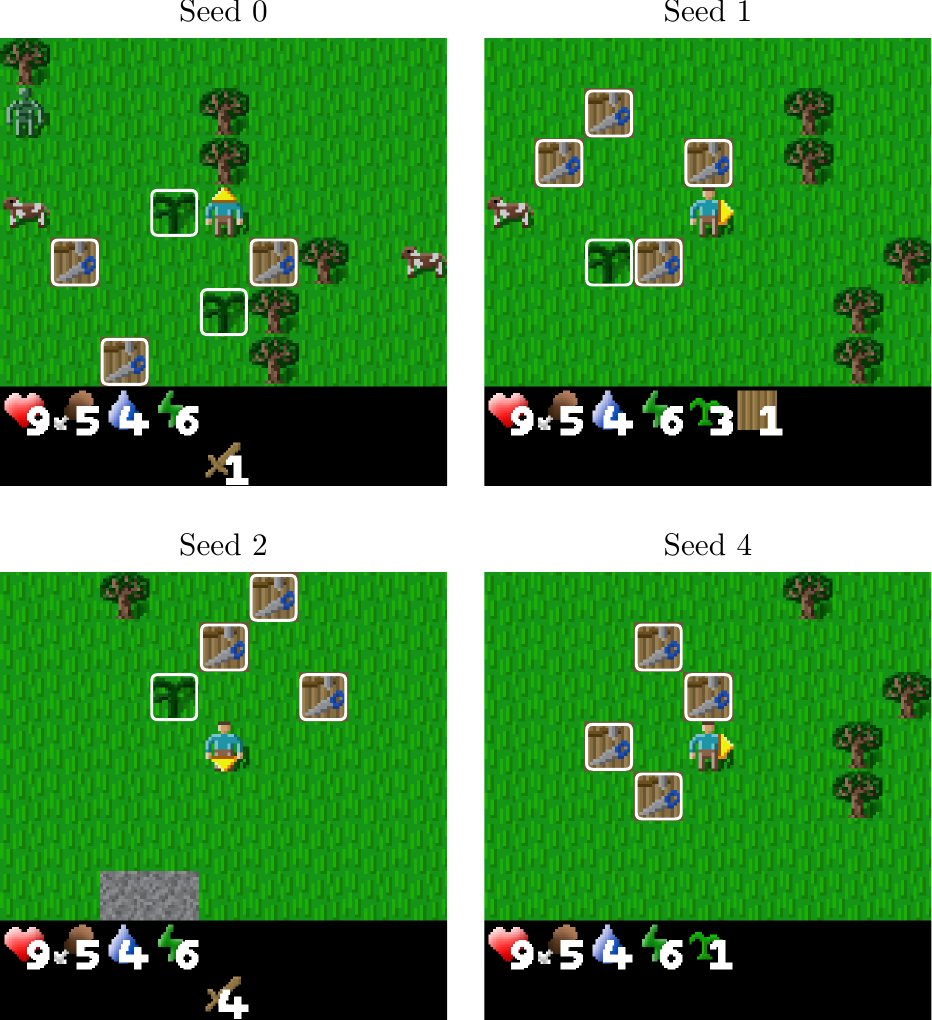}
    \caption{\textbf{Additional environment-modification behaviors.} Stills of environment modifications over the remaining $4$ seeds in Craftax-Classic. Each panel shows a selected rollout. White outlines box placed tiles.}
    \label{fig:building_examples}
\end{figure}

\section{Theory}
\label{app:theory}

Below, we include the key statement and proof that justifies the empowerment-maximization method. While the particular parameterizations of the underlying \gls{misl} methods to maximize potential empowerment are novel, we explicitly note that prior works such as VIC~\citep{gregor2017variational} have already shown versions of the same theoretical results. We include these statements for completeness and emphasize that the primary contribution of our work is an empirical demonstration of tractable empowerment maximization, rather than novel theoretical results.

\subsection{Policy improvement of the potential empowerment lower-bound}

The crux of (most) modern \gls{misl} methods is to learn a classifier that identifies skills from states, transitions, or trajectories, then use that classifier as a reward. In expectation, the cumulative reward estimates/lower bounds an \gls{mi} (see \cref{sec:learn_skills,app:effective_lb}). Policies that maximize this \gls{mi} become interpretable skills. In our method, we use the resulting \gls{mi} estimate as a reward, training a \textbf{potential empowerment policy} where agents navigate to states where they have many options. We emphasize that this result is a modification of the result proved in \cite{gregor2017variational}. We do not claim that this is a significant theoretical contribution, and include these results for completeness to show that, under assumptions, our method gives policy improvement. In practical settings with function approximation, these assumptions do not hold.

Here, we prove that alternating (1) 
fitting an on-policy skill classifier and (2) running policy improvement using the trained classifier as a reward improves a lower bound on the potential empowerment. The result requires two preliminaries. The first preliminary is the Barber-Agakov bound~\citep{barber2003im} which is a variational bound on \gls{mi} that comes from the positivity of the Kullback-Liebler Divergence; we have already invoked this bound when introducing the classifier rewards. The second preliminary is the statement that a standard \gls{mdp} admits greedy policy improvement~\citep{puterman2014markov}.

The proofs use the induced posterior $q_\phi$, where we drop the $\phi$ subscript for notational simplicity. Note that we can always recover an effective posterior from the learned classifier, since we know the baseline $p(z)$ exactly.

We fix a start state $s_0$ and a skill prior $p(z)$ shared across iterations. The discounted future-state distribution is
\[
p^\pi(s_+\mid s_0,z)=(1-\gamma)\sum_{k=0}^{\infty}\gamma^k
\Pr_\pi(S_{k+1}=s_+\mid S_0=s_0,Z=z).
\]
For continuous states, the probabilities are replaced by densities.

\makerestatable
\begin{theorem}[Policy improvement on a lower bound of potential empowerment]
    \label{thm:policy_improvement}
Let $0 < \gamma < 1$ be the discount factor and let $\pi^{(n)}(a \mid s, z; s_0)$ be the current skill-conditioned policy at iteration $n$, with $z \sim p(\dmy{z})$ sampled at $s_0$ and held fixed. Fit an on-policy classifier $q^{(n)}(z \mid s_0, s_+)$ to $p^{\pi^{(n)}}(z \mid s_0, s_+)$ and define the $(1-\gamma)$-normalized action value
\[
\begin{aligned}
Q_z^{(n)}(s,a; s_0)
&\triangleq Q_z^{\pi^{(n)},q^{(n)}}(s,a; s_0) \\
&\triangleq (1-\gamma)\,
\E_{\pi^{(n)}(\dmy{a}\mid \dmy{s}, z; s_0)}\left[
\sum_{k=0}^{\infty}\gamma^k
\log \frac{q^{(n)}(z \mid s_0, s_{t+k+1})}{p(z)}
\middle| s_t=s,\; a_t=a
\right].
\end{aligned}
\]
Let $\pi^{(n+1)}$ be the result of a greedy policy-improvement step with respect to $Q_z^{(n)}(\dmy{s},\dmy{a}; s_0)$ for every skill $z$ and start-state $s_0$. If the classifier is perfectly fit on-policy at each iteration, $q^{(n)}(z \mid s_0, s_+) = p^{\pi^{(n)}}(z \mid s_0, s_+)$, then
\[
I^{\pi^{(n+1)}}(Z; S_+ \mid S_0=s_0) \geq I^{\pi^{(n)}}(Z; S_+ \mid S_0=s_0)
\qquad \forall n
\]
which proves policy improvement under the assumptions of a perfectly-fit classifier.
\end{theorem}

Improving the \emph{estimated potential empowerment-seeking behavior} reduces to standard \gls{rl} with the same classifier intrinsic reward. Let $Q_z^{\pi,q}$ and $V_z^{\pi,q}$ denote values under policy $\pi$ and classifier $q$ where $n$ indexes training iterations,
and
$V_z^{\pi,q}(s;s_0)=\E_{a\sim\pi(\cdot\mid s,z;s_0)}[Q_z^{\pi,q}(s,a;s_0)]$. Then,
\begin{align}
    \widehat{\gEp}^{(n)}(s)
    &\approx \E_{p(\dmy{z})} \left[V_z^{\pi^{(n)},q^{(n)}}(s;s)\right] \\
    &\leq I^{\pi^{(n)}}(Z; S_+ \mid S_0 = s) \\
    &\leq \gEp(s) \\
    &= \max_{\pi \in \Pi_\cZ} I^\pi (Z;S_+ \mid S_0=s),
\end{align}
which, under the $\gamma$-discounted sum and these assumptions, is an unbiased Monte-Carlo estimate of a lower-bound on the cumulative potential empowerment objective:
\[
\E_{\pip}\!\left[\sum_{t\ge0}\gamma^t
    \E_{z\sim p}\!\left[V_z^{\pi^{(n)},q^{(n)}}(s_t;s_t)\right]\right]
    \leq \E_{\pip}\left[\sum_t \gamma^t \gEp(s_t)\right].
\]

\begin{proposition}[A fixed classifier reward satisfies standard \gls{mdp} conditions]
\label{prop:fixed_classifier_mdp}
Fix a classifier $q(z\mid s_0,s_+)$ and define the normalized intrinsic reward
\[
    r_q(s_0,z,s')\triangleq (1-\gamma)\log\frac{q(z\mid s_0,s')}{p(z)}.
\]
For any policy $\pi$, define
\[
Q_z^{\pi,q}(s,a;s_0)
\triangleq
\E_{\pi(\dmy{a}\mid \dmy{s}, z; s_0)}\left[\sum_{k\ge 0}\gamma^k r_q(s_0,z,s_{t+k+1})\;\middle|\; s_t=s,a_t=a\right].
\]
Then $Q_z^{\pi,q}$ satisfies the Bellman equation
\begin{align*}
Q_z^{\pi,q}(s,a;s_0)
&= (1-\gamma)\,\E_{p(\dmy{s'}\mid s,a)}\left[\log\frac{q(z\mid s_0,s')}{p(z)}\right] \\
&\quad + \gamma\,\E_{p(\dmy{s'}\mid s,a)}\left[V_z^{\pi,q}(s';s_0)\right],
\end{align*}
where $V_z^{\pi,q}(s;s_0)=\E_{\pi(\dmy{a}\mid s,z;s_0)}Q_z^{\pi,q}(s,a;s_0)$.
\end{proposition}

\begin{proof}
    Re-express the discounted sum in terms of the reward and remaining summation.
    The classifier is fixed, so $r_q$ is a stationary reward for the \gls{mdp} conditioned on skill and start-state $(z,s_0)$.
\end{proof}

\begin{proposition}[On-policy classifier gives tight lower bound on \gls{mi}]
\label{prop:ba_bound_tight}
For any policy $\pi$ and any distribution $q(z\mid s_0,s_+)$,
\begin{align*}
    \cL(\pi,q)
    &\triangleq \E_{p(\dmy{z})\,p^\pi_z(\dmy{s_+}\mid s_0)}
    \left[\log\frac{q(z\mid s_0,s_+)}{p(z)}\right] \\
    &\le I^\pi(Z;S_+\mid S_0=s_0).
\end{align*}
Equality holds if $q(z\mid s_0,s_+)=p^\pi(z\mid s_0,s_+)$.
\end{proposition}
\begin{proof}
    Follows from B-A bound.

\end{proof}
\begin{proof}[Proof of \cref{thm:policy_improvement}]
At iteration $n$, fit the classifier $q^{(n)}$ to the on-policy posterior:
\[
    \cL(\pi^{(n)},q^{(n)})=I^{\pi^{(n)}}(Z;S_+\mid S_0=s_0).
\]
Now hold $q^{(n)}$ fixed. By \cref{prop:fixed_classifier_mdp}, this is an ordinary \gls{mdp} with stationary intrinsic reward. A greedy policy-improvement step therefore gives
\[
    \cL(\pi^{(n+1)},q^{(n)})\ge \cL(\pi^{(n)},q^{(n)}).
\]
Then, refit the classifier to the new on-policy posterior. By \cref{prop:ba_bound_tight}, the refitted classifier tightens the lower bound which can only increase the \gls{mi} lower bound value:
\begin{align*}
I^{\pi^{(n+1)}}(Z;S_+\mid S_0=s_0)
&= \cL(\pi^{(n+1)},q^{(n+1)}) \\
&\ge \cL(\pi^{(n+1)},q^{(n)}) \\
&\ge \cL(\pi^{(n)},q^{(n)}) \\
&= I^{\pi^{(n)}}(Z;S_+\mid S_0=s_0).
\end{align*}
Thus exact on-policy classifier fitting plus greedy improvement monotonically improves the true \gls{mi}.
\end{proof}

Importantly, this theoretical result assumes an exact classifier and greedy policy updates at fixed $s_0$, which does not happen in practice with function approximation.

\printglossaries

\end{document}